%% file: neurips_2022.tex
\documentclass{article}

\usepackage[nonatbib, final]{neurips_2022}

\usepackage[utf8]{inputenc} % allow utf-8 input
\usepackage[T1]{fontenc}    % use 8-bit T1 fonts
\usepackage{hyperref}       % hyperlinks
\usepackage{url}            % simple URL typesetting
\usepackage{booktabs}       % professional-quality tables
\usepackage{amsfonts}       % blackboard math symbols
\usepackage{nicefrac}       % compact symbols for 1/2, etc.
\usepackage{microtype}      % microtypography
\usepackage{xcolor}         % colors
\usepackage{amsmath}

\DeclareMathOperator*{\argmin}{arg\,min}
\DeclareMathOperator*{\infimum}{inf}
\DeclareMathOperator*{\supremum}{sup}
\usepackage{amssymb}
\usepackage{mathtools}
\usepackage{amsthm}
\usepackage{algorithm}
\usepackage{algpseudocode}
\usepackage[symbol]{footmisc}

\usepackage{algcompatible}
\usepackage[pdftex]{graphicx}
\usepackage{float}
\usepackage{placeins}
\usepackage[square,numbers]{natbib}
\theoremstyle{plain}
\newtheorem{theorem}{Theorem}[section]
\newtheorem{proposition}[theorem]{Proposition}

\theoremstyle{definition}
\newtheorem{definition}[theorem]{Definition}

\theoremstyle{remark}

\renewcommand*{\theequation}{\theequationprefix}
\newcommand*{\theequationprefix}{}
\newcommand*{\equationprefix}[2][0]{%
  \renewcommand*{\theequationprefix}{#2}% Change the prefix
  \setcounter{equation}{#1}% reset the counter to optional argument
}

\newcommand{\algcomment}[1]{\hfill// {\small\texttt{#1}}}
\usepackage{tcolorbox}
\usepackage{caption}
\usepackage{subcaption}
\usepackage{sidecap}

\title{A Lagrangian Duality Approach to Active Learning}

\author{ Juan Elenter \thanks{Corresponding author.}   \\
University of Pennsylvania \\
\texttt{elenter@seas.upenn.edu} 
\And
Navid NaderiAlizadeh\\
University of Pennsylvania\\
\texttt{nnaderi@seas.upenn.edu}
	\And
Alejandro Ribeiro\\
University of Pennsylvania\\
\texttt{aribeiro@seas.upenn.edu}
}

\begin{document}

\maketitle

\begin{abstract}
We consider the pool-based active learning problem, where only a subset of the training data is labeled, and the goal is to query a batch of unlabeled samples to be labeled so as to maximally improve model performance. We formulate the problem using constrained learning, where a set of constraints bounds the performance of the model on labeled samples. Considering a primal-dual approach, we optimize the primal variables, corresponding to the model parameters, as well as the dual variables, corresponding to the constraints. As each dual variable indicates how significantly the perturbation of the respective constraint affects the optimal value of the objective function, we use it as a proxy of the informativeness of the corresponding training sample. Our approach, which we refer to as Active Learning via Lagrangian dualitY, or ALLY, leverages this fact to select a diverse set of unlabeled samples with the highest estimated dual variables as our query set. We demonstrate the benefits of our approach in a variety of classification and regression tasks and discuss its limitations depending on the capacity of the model used and the degree of redundancy in the dataset. We also examine the impact of the distribution shift induced by active sampling and show that ALLY can be used in a generative mode to create novel, maximally-informative samples.
\end{abstract}

% keywords can be removed
%\keywords{Active Learning \and Lagrangian Duality \and Constrained Learning \and Convex Optimization}

\section{Introduction}
\label{intro}

One of the key drivers of the recent progress in machine learning is the availability of massive, high-quality datasets, which enables training models comprising millions, or even billions, of parameters~\citep{GPT3, lepikhin2021gshard, dosovitskiy2021an}. Nevertheless, in some areas, such as healthcare, obtaining \emph{labeled} training data is challenging and/or expensive~\citep{al_cancerdet, al_medical, med_im_seg}. This has given rise to a class of approaches, collectively referred to as \emph{active learning}, whose goal is to minimize the labeling effort for training machine learning models.

Active learning methods aim to improve data efficiency by querying the labels of samples presumed informative, in a feedback-driven fashion. In recent years, the pool-based active learning setting, in which queries are drawn from a large, static pool of unlabeled samples has drawn significant attention~\citep{al_survey, CORESET, bal_at_scale}. This is due to the abundance of such unlabeled pools and the compatibility of the pool-based setting with the training of deep neural networks. 

Most active learners rely on defining a notion of \emph{informativeness} of a given sample, such as model uncertainty \citep{poolbased, sapiroal}, expected model change \citep{emc, maxemc} or expected error reduction \citep{eer, BADGE}. However, in the batch setting, where multiple samples are queried simultaneously, not contemplating the information overlap between the samples can lead to sub-optimal queries. Consequently, batch \emph{diversity} needs to be taken into account, often at the expense of individual sample informativeness \citep{diversity}.

%Although a myriad of active learning algorithms have been proposed in the literature, the majority of them come with a number of limitations. For instance, most uncertainty measures are not compatible with regression scenarios and can only work in regression settings~\citep{xx}. Moreover, many approaches rely on heuristic measures for identifying informative samples, which are not grounded in theoretical analysis~\citep{xxx}. \nn{anything else? we will also need to add citations for the points above}

% many share strong limitations. \je{Decrease the arrogance :)} For instance, most uncertainty measures only work in classification settings. -unprincipled, -inconsistent, -not compatible with deep neural nets

In this paper, we demonstrate how a \emph{constrained learning} formulation of the problem enables the use of \emph{Lagrangian duality} for detecting informative samples. In particular, we bound the loss incurred by each sample, and use the dual variables associated to these constraints as a measure of informativeness. We show that dual variables are directly related to the variations of the average optimal loss over the entire data distribution, which motivates our approach.

Through an iterative primal-dual strategy, we optimize the model parameters as well as the dual variables. We then leverage the learned embedding space~\citep{bengio2013representation, oord2018representation} to train a \emph{dual regression head} that estimates the dual variable associated to each unlabeled sample. Our proposed active learning approach, which we refer to as Active Learning via Lagrangian dualitY, or ALLY can then be used to select a diverse and informative set of samples. We also argue that, under certain conditions, the strategy put forward is not severely impacted by the distribution shift induced by active sampling. 

%selects a diverse batch by clustering the unlabeled samples in the embedding space using $k$-MEANS \citep{kmeans}, and uniformly selecting those with highest associated dual variable from each cluster. %to enhance diversity among the selected samples.

%Through an iterative primal-dual approach, . 

We evaluate the performance of ALLY on a suite of classification and regression tasks, and show that it performs similar to or better than state-of-the-art batch active learning methods. We further demonstrate how the trained backbone, alongside the dual regression head, enable the generation of novel samples that can be optimized to be maximally informative, shedding light on the interpretability of the proposed active learning framework. % \je{love the ending, but would be a bit more conservative with the superiority part, the gains are small and not in all scenarios.}

%In summary, this paper puts forward a principled and novel active learning method based on Lagrangian duality. The method works for classification as well as regression and is agnostic to the underlying architecture. 

\section{Related Work}
\subsection{Active Learning}
\label{poolbasedal}

The literature on active learning is voluminous and a myriad of strategies for the pool-based setting have been proposed \citep{al_survey}. In what follows, we describe some of the approaches most connected to our work.

A simple way of measuring model uncertainty is by computing the entropy of the predicted class distribution. Designed for the sequential case, Entropy Sampling \citep{poolbased} selects the unlabeled sample with highest associated output entropy. Among the relevant methods is BADGE \citep{BADGE}, which employs a lower bound on the norm of the gradients in the final layer of the network as a measure of informativeness. BADGE balances diversity  and informativeness by using the $k$-MEANS++ seeding algorithm to select a batch with large Gram determinant in the gradient space. BAIT \citep{BAIT} builds on traditional, Information Matrix based methods to efficiently select a batch that optimizes a bound on the MLE error in a two stage manner. Other methods that propose notions of informativeness are BALD \citep{BALD}, which uses the mutual information between predictions and model parameters as an uncertainty measure, and Learning Loss \citep{learninglossal}, which trains a loss prediction module and queries the samples that hypothetically generate high errors (and, thus, large model updates). VAAL \citep{VAAL} and WAAL \citep{WAAL} are methodologically similar to \citep{learninglossal} in the sense that they all use a multipartite system consisting of a feature encoder, a prediction head and an auxiliary estimator. This is also the case for the strategy presented in this paper.

Some diversity-promoting approaches are compatible with many informativeness measures. A popular approach is to cluster the samples of the unlabeled set and then select informative points from each cluster \citep{div_kmeans1, div_kmeans2, div_kmeans3}. In \citep{div_approx}, Monte Carlo sampling is used to simulate sequences of length $b$ of the sequential algorithm, and then a \textit{best-matching} combination of the sequences is used to build a batch. A simpler approach is to select the $b$ most informative points after a stochastic perturbation of the informativeness scores \citep{stoch_db}.

Some methods do not enforce informativeness explicitly, but rather query a set of data points that is maximally representative of the entire unlabeled set. Coreset \citep{CORESET}, for instance, formulates pool-based active learning as a core-set selection problem, and aims to identify a set of points that geometrically covers the entire representation space. To do this, Coreset selects the batch that when added to the labeled set, minimizes the maximum distance between labeled and unlabeled examples. Coreset is compatible with deep neural networks and can be used in both regression and classification settings. Similarly, DAL \citep{disc_active_learning} emphasizes representativeness by framing active learning as a binary classification task and selecting queries that maximize the similarity between the labeled and unlabeled set.

\subsection{Constrained Learning}

The need to tailor the behavior of machine learning systems has led to the development of a constrained learning theory. A common approach is to use regularization, that is, to modify the learning objective so as to promote certain requirements. However, choosing the level of regularization through a dimensionless weight can be more challenging than setting a constraint level.
Furthermore, recent works \citep{LuizPAC, LuizNonConvex, lagrangianduality2} show that, from a PAC (Probably Approximately Correct) perspective, learning under requirements is as hard as classical learning and that it can be done in practice through primal-dual learners. This has led to numerous applications across several areas of ML such as federated learning \citep{pochi1}, fairness \citep{LuizPAC}, stability of neural networks \citep{pochi2, raghu}, adversarial robustness \citep{robey} and data augmentation \citep{chunis}. 

\section{Problem Formulation}

\subsection{Batch Active Learning}

Let $\mathfrak{D}$ denote a probability distribution over data pairs $(\boldsymbol{x}, y)$, where $\boldsymbol{x} \in \mathcal{X} \subseteq \mathbb{R}^D$  represents a feature vector (e.g., the pixels of an image) and $y \in \mathcal{Y} \subseteq \mathbb{R}$ represents a label or measurement. In classification tasks, $\mathcal{Y}$ is a subset of $\mathbb{N}$, whereas in regression, $\mathcal{Y} = \mathbb{R}$.

Initially, a set $\mathcal{L} = \{(\boldsymbol{x}_i, y_i)\}_{i \in \mathcal{N}_{\mathcal{L}}}$ of data pairs, or labeled samples, is available, coming from a probability distribution $\mathfrak{D}$. This set is used to learn a predictor $f(\mathbf{x}; \mathcal{L}): \mathcal{X} \to \mathcal{Y}$ from a hypothesis class $\mathcal{F}$. Note that dependence on the set used to learn $f$ is made explicit. Then, a batch $\mathcal{B}$ of samples, or \emph{queries}, is selected from a pool of unlabeled samples $\mathcal{U} = \{\boldsymbol{x}_i \}_{i \in \mathcal{N}_\mathcal{U}} $ and sent to an oracle for labeling. The goal is to the select the batch that minimizes the expected loss over the natural data distribution. More precisely, we formulate the Batch Active Learning (BAL) problem as
\begin{align}
\label{eq:batch_activelearning}
\tag{BAL}
\mathcal{B}^{\star} = \argmin _{\mathcal{B} \subseteq \mathcal{U} \: : \: | \mathcal{B} | \leq b} \; \min _{f \in \mathcal{F}} \mathbb{E}_{(\mathbf{x}, y) \sim \mathfrak{D}}\left[\ell\left(f(\mathbf{x; \mathcal{L}} \cup \mathcal{B}), y%; \mathcal{A}_{{\mathcal{L} \cup \mathcal{B}}}
\right)\right],
\end{align}
where $b$, referred to as the \emph{budget}, represents the maximum cardinality of $\mathcal{B}$ and $\ell:\mathcal{Y} \times \mathcal{Y} \to \mathbb{R}$ is a loss function (e.g., cross-entropy loss or mean-squared error). 

This process is typically repeated multiple times. At each iteration, two main steps are performed: (i) selecting $\mathcal{B}$ and updating the sets: $\mathcal{L}^{(t)} = \mathcal{L}^{(t-1)} \cup \mathcal{B}$ and $\mathcal{U}^{(t)} = \mathcal{U}^{(t-1)} \setminus \mathcal{B}$, and (ii) obtaining the predictor $f$ with the aggregate set of labeled samples $\mathcal{L}^{(t)}$. Steps (i) and (ii) correspond to the outer and inner minimization problems in~\eqref{eq:batch_activelearning}, respectively. In what follows, we focus on a single iteration, and thus obviate the dependence on the iteration $t$ and the set used to ease the notation.

\subsection{Constrained Statistical Learning}\label{sec:CSL}

Most active learning methods in the literature~\citep{CORESET, BADGE, BALD, al_medical, al_survey} formulate step (ii) above as an \emph{unconstrained} Statistical Risk Minimization (SRM) problem \citep{VapnikSRM},
\begin{align}
\label{eq:SRM}
\tag{SRM}
\min_{f \in \mathcal{F}} \: \mathbb{E}_{(\boldsymbol{x}, y) \sim \mathfrak{D}}\left[\ell \left(f(\boldsymbol{x}), y\right)\right].
\end{align}
% Note that problem \ref{eq:SRM} is at the core of modern machine learning.
Our approach, alternatively, uses a \emph{Constrained} Statistical Learning (CSL) formulation,
\equationprefix{CSL} % must be before subequation!
\begin{subequations}\label{eq:cl_formulation}
\renewcommand{\theequation}{CSL-\alph{equation}}
\begin{alignat}{2}
    P^{\star} = &\min_{f \in \mathcal{F}}  &\quad& \mathbb{E}_{(\boldsymbol{x}, y) \sim \mathfrak{D}}\left[\ell \left(f(\boldsymbol{x}), y\right)\right]\label{eq:CSL_obj} \\
    &~~\text{s.t.} && \ell' \left(f(\boldsymbol{x}), y\right) \leq \epsilon(\boldsymbol{x}), \quad \mathfrak{D}_{\boldsymbol{x}} \text {--a.e. }\label{eq:CSL_const}
\end{alignat}
\end{subequations}
\renewcommand{\theequation}{\arabic{equation}}%
\equationprefix{}% no prefix
where $\ell':\mathcal{Y} \times \mathcal{Y} \to \mathbb{R}$ is a secondary loss function, $\epsilon: \mathcal{X} \to \mathbb{R}$ is a mapping from each data point to a corresponding constraint upper bound, and $\mathfrak{D}_{\boldsymbol{x}}$ denotes the marginal distribution over $\mathcal{X}$. Note that the objective function in~\eqref{eq:CSL_obj} is the same as in~\eqref{eq:SRM}, but the %\emph{per-sample}
secondary loss is required to be bounded $\mathfrak{D}_{\boldsymbol{x}}-$almost everywhere.

Letting $\lambda:\mathcal{X} \to \mathbb{R}^+$ denote the dual variable function, the Lagrangian associated to~\eqref{eq:cl_formulation} can be written as
\begin{align*}
L(f, \lambda) &= \mathbb{E}_{(\boldsymbol{x}, y) \sim \mathfrak{D}} \: [\ell (f(\boldsymbol{x}), y) ] \quad+
\int_{\mathcal{X,Y}} \lambda(\boldsymbol{x})(\ell' (f(\boldsymbol{x}), y)-\epsilon(\boldsymbol{x}))p(\boldsymbol{x}, y) d \boldsymbol{x} d y \\
& =  \mathbb{E}_{(\boldsymbol{x}, y) \sim \mathfrak{D}} \: \Big[ \ell (f(\boldsymbol{x}), y)  + \lambda(\boldsymbol{x})( \ell'(f(\boldsymbol{x}), y)  - \epsilon(\boldsymbol{x}) )\Big],
\end{align*}
where it is implicitly assumed that the conditional distribution $p(y|\boldsymbol{x})$ is a Dirac delta distribution, i.e., $y$ is a deterministic function of $\boldsymbol{x}$. This leads to the dual problem,
\begin{align}
\label{eq:D-CSL}
\tag{D-CSL}
D^{\star}=\max _{\lambda\in\Lambda} \; \min _{f \in \mathcal{F}} L\left(f, \lambda(\boldsymbol{x}) \right),
\end{align}
where %$\Lambda\coloneqq\{\lambda | \lambda:\mathcal{X} \to \mathbb{R}^+ \}.$
$\Lambda\coloneqq\{\lambda  \, | \, \lambda(\boldsymbol{x}) \geq 0, \:\: \mathfrak{D}_{\boldsymbol{x}} \text{--a.e.} \}$. There are three main motivations for this infinite programming formulation:
\begin{enumerate}
    \item \textbf{Access to variations of $P^{\star}$}: As we show in Theorem~\ref{theo:obj_derivative}, this formulation gives us access to $ \partial P^{\star} ( \epsilon(\boldsymbol{x}))$, enabling the use of \emph{dual variables} as an indicator of the \emph{informativeness} of the training samples. 
  % \item \textbf{Resilience}:  
    \item \textbf{Considering both average and worst-case losses}: The most informative samples often lie in the tails of the distribution $\mathfrak{D}$. Those samples appear less frequently in the dataset and thus, models obtained by solving \eqref{eq:SRM} can achieve low \emph{average} errors without learning to classify/regress them correctly. Our formulation bridges the gap between the \eqref{eq:SRM} and a worst-case Feasibility Statistical Learning (FSL) formulation,
%     \begin{align}
% \label{eq:FSL}
% % \tag{FSL}
%  & P^{\star} = \min_{f \in \mathcal{F}}  \quad 0   \\
%  &\text{s.t.} \quad \ell \left(f(\boldsymbol{x}), y\right) \leq \epsilon(\boldsymbol{x}), \quad \mathfrak{D}_{\boldsymbol{x}} \text {--a.e. }   \end{align} 
    \equationprefix{FSL} % must be before subequation!
    \begin{subequations}\label{FSL}
    \renewcommand{\theequation}{FSL-\alph{equation}}
    \begin{alignat}{2}
        P^{\star} = &\min_{f \in \mathcal{F}}  &\quad& 0 \\
        &~~\text{s.t.} && \ell \left(f(\boldsymbol{x}), y\right) \leq     \epsilon(\boldsymbol{x}), \quad \mathfrak{D}_{\boldsymbol{x}}     \text {--a.e. }
    \end{alignat}
    \end{subequations}
    \renewcommand{\theequation}{\arabic{equation}}%
    \equationprefix{}% no prefix
    which considers the worst-case loss.
    \item \textbf{Adaptive regularization}: For a fixed dual variable $\lambda$, the Lagrangian is a regularized objective, where $\ell'$ acts as a regularizing functional. Thus, the max-min formulation in \ref{eq:D-CSL} can be viewed as a regularized minimization, where the regularization weight is updated during the training procedure according to the degree of constraint satisfaction or violation. 
    %\item Bridging the gap between feasibility and 
\end{enumerate}

The dual problem can be interpreted as finding the tightest lower bound on $P^{\star}$. In the general case, $D^{\star} \leq P^{\star}$, which is known as weak duality. Nevertheless, under certain conditions, $D^{\star}$ attains $P^{\star}$ (strong duality) and we can derive a relation between the solution of \eqref{eq:D-CSL} and the sensitivity of $P^{\star}$ with respect to $\epsilon(\boldsymbol{x})$. See Appendix~\ref{appx:duality} for more details.

In the following, we define the \emph{Fréchet subdifferential} of a convex function, which allows us to justify the use of dual variables as a measure of \emph{informativeness} of a sample.

\begin{definition}
Let $U, V$ be Banach spaces. The Fréchet subdifferential of a functional $P: U \to V $ at $u \in U$ is defined as:
$$
\partial P(u) = \{ z \in U^* \, : \, P(v) - P(u) \geq \langle z \, , \, v-u \rangle \: \: \text{for all $v \in U$}\},
$$

where $U^*$ denotes the topological dual space of $U$, and $\langle z \, , \, v-u \rangle = \mathbb{E}_{\mathfrak{D}}\left[ z(\boldsymbol{x}) (v(\boldsymbol{x}) - u(\boldsymbol{x})) \right] $.
\end{definition}

Having the above definition, we state following theorem, which characterizes the variations of $P^{\star}$ (the optimum value of \ref{eq:cl_formulation}) as a function of the constraint tightness $\epsilon(\boldsymbol{x})$.

\begin{theorem}\label{theo:obj_derivative}
If the problem~\eqref{eq:cl_formulation} is strongly dual, then for any $\boldsymbol{x} \in \mathcal{X}$, we have
\begin{equation*}
% \label{prop1}
  - \lambda^{\star}(\boldsymbol{x}) \: \in \: \partial P^{\star}(\epsilon(\boldsymbol{x})), 
\end{equation*}
where $\partial P^{\star}(\epsilon(\boldsymbol{x}))$ denotes the Fréchet subdifferential of $P^{\star}$ with respect to $\epsilon(\boldsymbol{x})$, and $\lambda^{\star}(\boldsymbol{x})$ is the optimal dual variable associated to the constraint on $\boldsymbol{x}$. 
\end{theorem}

\begin{proof}
See Appendix~\ref{appx:theo_proof}.
\end{proof}

%In the infinite program~\eqref{eq:cl_formulation}, there is one constraint associated to each $\boldsymbol{x_0} \in \mathcal{X}$. 
For any $\boldsymbol{x_0} \in \mathcal{X}$, let $\delta_{\boldsymbol{x_0}}(\boldsymbol{x}) $ be a bump function of radius $r > 0$, centered at $\boldsymbol{x_0}$ (i.e., a continuous, radially-increasing function with support in $\|\boldsymbol{x}-\boldsymbol{x_0}\| \leq r$). Theorem \ref{theo:obj_derivative} implies that
\begin{align*}
P^{\star}(\epsilon(\boldsymbol{x})+ \delta_{\boldsymbol{x_0}}(\boldsymbol{x})) - P^{\star}(\epsilon(\boldsymbol{x})) \geq \langle \, -\lambda^{\star}(\boldsymbol{x_0}) , \,\delta_{\boldsymbol{x_0}}(\boldsymbol{x}) \, \rangle.
\end{align*}

The problem~\eqref{eq:cl_formulation} typically includes an infinite number of constraints. Theorem~\ref{theo:obj_derivative} implies that the constraint whose perturbation has the most \emph{potential impact} on the optimal value of~\eqref{eq:cl_formulation} is the constraint with the highest associated optimal dual variable. %For instance, infinitesimally relaxing the constraint in a neighbourhood $\boldsymbol{x_0}$ extends the feasible set, leading to a decrease of the optimal value of \eqref{eq:cl_formulation} at a rate larger than $\lambda^{\star}(\boldsymbol{x_0})$. 
For instance, infinitesimally tightening the constraint in a neighbourhood $\boldsymbol{x_0}$ would restrict the feasible set, causing an increase of the optimal value of \eqref{eq:cl_formulation} at a rate larger than $\lambda^{\star}(\boldsymbol{x_0})$. In that sense, the magnitude of the dual variables can be used as a measure of informativeness. Similarly to non-support vectors in SVMs \citep{svmvapnik}, samples associated to inactive constraints (i.e., \{$\boldsymbol{x_0} : \lambda^{\star}(\boldsymbol{x_0}) = 0\}$), are considered uninformative. 

\subsection{On the Statistical Bias Induced by Active Sampling}
\label{sec:biasAL}
In pool-based active learning, the resulting training set may not be representative of the natural data distribution $\mathfrak{D}$ \citep{biasinAL1, biasinAL2, biasinAL3}. This is due to the fact that queried samples are not randomly selected and often lie in the tails of $\mathfrak{D}$. Therefore, when performing several active learning iterations, we undertake a biased version of~\eqref{eq:batch_activelearning}:
\begin{align}
\label{biasedBAL}
\tag{bBAL}
\mathcal{B}^{\star} = \argmin _{\mathcal{B} \subseteq \mathcal{U}^{(t)} \: : \: | \mathcal{B} | \leq b} \; \min _{f \in \mathcal{F}} \mathbb{E}_{(\mathbf{x}, y) \sim \mathfrak{A}^{(t)}}\left[\ell\left(f(\mathbf{x}; \mathcal{L}^{(t)} \cup \mathcal{B}), y
\right)\right],
\end{align}
where $\mathfrak{A}^{(t)}$ represents the biased distribution underlying the actively sampled set $\mathcal{L}^{(t)}$. Even though, at each iteration, the queried samples have some desired property (e.g., large impact on the expected loss), the learned predictor $f(\mathbf{x}; \mathcal{L}^{(t)} \cup \mathcal{B})$ is usually sub-optimal for the natural data distribution.

One way of undertaking the inner minimization in \eqref{biasedBAL} is with the worst-case/feasibility formulation,
\equationprefix{bFSL} % must be before subequation!
\begin{subequations}\label{bFSL}
\renewcommand{\theequation}{bFSL-\alph{equation}}
\begin{alignat}{2}
    P^{\star} = &\min_{f \in \mathcal{F}}  &\quad& 0 \\
    &~~\text{s.t.} && \ell \left(f(\boldsymbol{x}), y\right) \leq \epsilon(\boldsymbol{x}), \quad \mathfrak{A}^{t}_{\boldsymbol{x}} \text {--a.e. }
\end{alignat}
\end{subequations}
\renewcommand{\theequation}{\arabic{equation}}%
\equationprefix{}% no prefix
This formulation is closely related to the constrained formulation in \eqref{eq:cl_formulation}. In Appendix \ref{app:biasAL}, we relate the dual problem of (bFSL) with \eqref{eq:D-CSL} to argue that the impact of the inconsistency between distributions on our method is small. The key observation is that, if the marginal distributions $\mathfrak{D}_{\mathbf{x}}$ and $\mathfrak{A}^{(t)}_{\mathbf{x}}$ have the same support, then the respective feasibility formulations are equivalent, regardless of the probability distribution. %This indicates that the informativeness scores $\lambda^{\star}(\mathbf{x})$ that stem from solving \eqref{eq:D-CSL} are not significantly impacted by the statistical bias induced by active sampling.

%This is summarized in the following proposition:

%\textbf{Proposition 3.3: }  Let (D-FSL) be the dual problem associated the feasibility formulation \eqref{eq:FSL} and recall that \eqref{eq:D-CSL} is the original dual problem. If $\mathfrak{D}$ and $\mathfrak{D^'}$ satisfy $$\{ x \: : \: p_{\mathfrak{D}_{\boldsymbol{x}}}(\mathbf{x}) > 0 \} =  \{ x \: : \: p_{\mathfrak{D'}_{\boldsymbol{x}}}(\mathbf{x}) > 0 \}$$ then the problems (D-FSL) and \eqref{eq:D-CSL} are equivalent. %In consequence, the informativeness scores $\lambda^*(\mathbf{x})$ that stem from solving \eqref{eq:D-CSL} are invariant to the distribution shift.

\section{Proposed Approach}
\label{proposed_app}

In light of the results mentioned in Section~\ref{sec:CSL} on the usefulness of dual variables in constrained statistical learning as a measure of sample informativeness, we present our proposed method, ALLY, in this section. We start by introducing a primal-dual procedure to empirically solve the constrained learning problem, and we will then proceed to describe our active learning algorithm in detail.

\subsection{Constrained Empirical Risk Minimization}\label{empirical_csl}

The formulation in~\eqref{eq:D-CSL} poses two challenges: (i) the distribution $\mathfrak{D}$ is usually unknown and (ii) it is an infinite-dimensional problem since it optimizes over the functional spaces $\mathcal{F}$ and $\Lambda$. We handle the former by replacing expectations by their sample means over a set of labeled samples $\mathcal{L} = \{(\boldsymbol{x}_i, y_i)\}_{i \in \mathcal{N}_{\mathcal{L}}}$, as described in the classical Empirical Risk Minimization (ERM) theory \citep{VapnikSRM, understandingML}. In order to resolve the latter, we introduce a \emph{parameterization} of the hypothesis class $\mathcal{F}$ as $\mathcal{P} = \{f_{\boldsymbol{\theta}} \, | \, \boldsymbol{\theta} \in \Theta \}$, while we create a separate dual variable $\lambda_i$ for each training sample $\boldsymbol{x}_i$. These modifications lead to the Constrained Empirical Risk Minimization (CERM) problem,
\equationprefix{CERM} % must be before subequation!
\begin{subequations}\label{eq:cerm_formulation}
\renewcommand{\theequation}{CERM-\alph{equation}}
\begin{alignat}{2}
    \hat{P^{\star}} = &\min_{\boldsymbol{\theta} \in \Theta}  &\quad& \frac{1}{|\mathcal{N}_{\mathcal{L}}|}\sum_{i \in \mathcal{N}_{\mathcal{L}}} \ell \left(f_{\boldsymbol{\theta}}(\boldsymbol{x_i}), y_i \right) \\
    &~~\text{s.t.} && \ell' \left(f_{\boldsymbol{\theta}}(\boldsymbol{x_i}), y_i\right) \leq \epsilon_i, \: \forall i \in \mathcal{N}_{\mathcal{L}}.
\end{alignat}
\end{subequations}
\equationprefix{}%
This, in turn, results in the corresponding empirical dual problem,
\begin{align}
\hat{D^{\star}}=\max_{\boldsymbol{\lambda} \geq \mathbf{0}} \; \min _{\boldsymbol{\theta} \in \Theta} \hat{L}(\boldsymbol{\theta}, \boldsymbol{\lambda}), \tag{D-CERM}\label{eq:D-CERM}
\end{align}
where $\boldsymbol{\lambda}=\{\lambda_i\}_{i\in\mathcal{N}_{\mathcal{L}}}$, $\boldsymbol{\lambda} \geq \mathbf{0}$ represents element-wise non-negativity, $\epsilon_i$ denotes the constraint upper bound associated to the $i$\textsuperscript{th} point-wise constraint, and the \emph{empirical} Lagrangian, $\hat{L}(\boldsymbol{\theta}, \boldsymbol{\lambda})$, is defined as
\begin{align*}
\hat{L}(\boldsymbol{\theta}, \boldsymbol{\lambda}) = \frac{1}{|\mathcal{N}_{\mathcal{L}}|} \, \sum_{i \in \mathcal{N}_{\mathcal{L}}} \: &\Big[\ell (f_{\boldsymbol{\theta}}(\boldsymbol{x_i}), y_i) + \lambda_i \, [\ell' (f_{\boldsymbol{\theta}}(\boldsymbol{x_i}), y_i)-\epsilon_i)]\Big].
\end{align*}

The max-min problem~\eqref{eq:D-CERM} can be undertaken by alternating the minimization with respect to $\boldsymbol{\theta}$ and the maximization with respect to $\boldsymbol{\lambda}$  \citep{arrowhurwitz, LuizPAC, lagrangianduality2}, which leads to the primal-dual constrained learning procedure in Algorithm \ref{alg:pd}. Notice that $\min _{\boldsymbol{\theta} \in \Theta} \hat{L}(\boldsymbol{\theta}, \boldsymbol{\lambda})$ is the minimum of a family of affine functions on $\boldsymbol{\lambda}$, and thus is concave. Consequently, the outer problem corresponds to the maximization of a concave function and can be solved via gradient ascent. The inner minimization, however, is generally non-convex, but there is empirical evidence that deep neural networks can attain \textit{good} local minima when trained with stochastic gradient descent \citep{rethink_gen, arpit2017closer}. Some theoretical remarks on Algorithm \ref{alg:pd} can be found in Appendix \ref{app:convergence}. %\textbf{Convergence of Algorithm 1.}

%Convergence is guaranteed in the sense that the final parameters $\boldsymbol{\theta}$ and dual variables $\boldsymbol{\lambda}$ returned by Algorithm~\ref{alg:pd} satisfy
% \vspace{-.1in}
%$$ 
%P^{\star} - k_1 \leq \hat{L}(\boldsymbol{\theta}, \boldsymbol{\lambda}) \leq %P^{\star} + k_2,
% \vspace{-.06in}
%$$
%where $k_1$ and $k_2$ are positive constants that depend on the losses, the number of samples, and the richness of the parameterization. For a thorough analysis on this bound, the reader is referred to \citep{LuizPAC}.

As shown in Algorithm~\ref{alg:pd}, the dual variables accumulate the slacks (i.e., distances between the per-sample secondary loss and constraint values) over the entire learning procedure. This allows the dual variables to be used as a measure of informativeness, while at the same time affecting the local optimum to which the algorithm converges. Quite interestingly, similar ideas on monitoring the evolution of the loss for specific training samples in order to recognize impactful instances have been used in several generalization analyses~\citep{forgetting, notallsamples}.

\begin{algorithm}[t]
\caption{Primal-dual constrained learning (PDCL)}
    \label{alg:pd}
    \begin{algorithmic}[1]
    \STATE {\bfseries Input:} Labeled dataset $\mathcal{L}$, primal learning rate $\eta_p$, dual learning rate $\eta_d$, number of iterations $T$, number of primal steps per iteration $T_p$, constraint vector $\boldsymbol{\epsilon}$.
    \STATE Initialize: $\boldsymbol{\theta}$,  $\boldsymbol{\lambda}\leftarrow\mathbf{0}$.
    \FOR{$t=1, \ldots, T$}
        \STATE 
        $\boldsymbol{\theta} \: \leftarrow \: \boldsymbol{\theta} - \eta_p \nabla_{\theta}\hat{L}(\boldsymbol{\theta}, \boldsymbol{\lambda}) \quad (\times T_p)
        $ \algcomment{Update primal variables ($T_p$ SGD steps)}
        \STATE
        $
        s_{i} \leftarrow \ell' \left(f_{\boldsymbol{\theta}}\left(\boldsymbol{x}_{i}\right), y_{i}\right)-\epsilon_i, ~\forall i\in\mathcal{N}_{\mathcal{L}}.
        $ \algcomment{Evaluate constraint slacks}
        \STATE
        $
        \lambda_{i} \leftarrow \left[\lambda_{i} + \eta_d s_{i} \right]_{+}, ~\forall i\in\mathcal{N}_{\mathcal{L}}.
        $ \algcomment{Update dual variables}
    \ENDFOR
    \STATE {\bfseries Return:}{  $\boldsymbol{\theta}$, $\boldsymbol{\lambda}$.} 
    \end{algorithmic}
\end{algorithm}

\subsection{ALLY: Active Learning via Lagrangian Duality}

Our proposed active learning algorithm, ALLY, is presented in Algorithm \ref{alg:ally} (for the case of $b=1$). Given a set of labeled samples, $\mathcal{L}$, we first obtain the model parameters $\boldsymbol{\theta}^{\star}$ and the dual variables associated to samples in $\mathcal{L}$ using the primal-dual constrained learning approach in Algorithm~\ref{alg:pd}. Taking a representation learning approach~\citep{bengio2013representation, oord2018representation, tian2020contrastive, chen2020simple}, we then partition the model $f_{\boldsymbol{\theta}^{\star}}$ to a \emph{backbone} $f_{\boldsymbol{\phi}^{\star}}:\mathcal{X} \to \mathbb{R}^d$, where $d$ denotes the dimensionality of the \emph{embedding space}, and a \emph{prediction head} $f_{\boldsymbol{\psi}^{\star}}: \mathbb{R}^d \to \mathcal{Y}$, such that $f_{\boldsymbol{\theta}^{\star}} = f_{\boldsymbol{\phi}^{\star}} \circ f_{\boldsymbol{\psi}^{\star}}$ and $\boldsymbol{\theta}^{\star} = \boldsymbol{\phi}^{\star} \cup \boldsymbol{\psi}^{\star}$. In order to estimate the informativeness of the samples in the unlabeled dataset $\mathcal{U}$, we train a \emph{dual regression head} $f_{\boldsymbol{\omega}}: \mathbb{R}^d \to \mathbb{R}^+$ on the embeddings generated by $f_{\boldsymbol{\phi}^{\star}}$ by minimizing the mean-squared error
\begin{align}\label{eq:lambda_mse}
L_{\boldsymbol\lambda}(\boldsymbol{\omega}) = \frac{1}{|\mathcal{N}_{\mathcal{L}}|} \sum_{i\in\mathcal{N}_{\mathcal{L}}} \left\|f_{\boldsymbol{\omega}}(f_{\boldsymbol{\phi}^{\star}}(\boldsymbol{x}_i)) - \lambda_i^{\star} \right\|^2,
\end{align}
while the parameters $\boldsymbol{\phi}^{\star}$, hence the embeddings, are kept frozen. It should be noted that the idea of mapping embeddings to dual variables is present in other machine learning settings \citep{vec2lambda}. Once the dual regression head is trained, we evaluate it on the embeddings corresponding to the unlabeled samples, and identify the sample with the highest predicted dual variable.

As explained in Section~\ref{poolbasedal}, in the batch setting, selecting the $b$ samples with the highest associated dual variables is not optimal, due to the potential information overlap of such samples \citep{al_survey, diversity}. Our method is compatible with any batch diversity approach that takes informativeness scores and unlabeled data points (or embeddings) as inputs.

\begin{algorithm}[t]
  \caption{Active learning via Lagrangian duality (ALLY)}
  \label{alg:ally}
\begin{algorithmic}[1]
  \STATE {\bfseries Input:} Labeled set $\mathcal{L}$,  unlabeled set $\mathcal{U}$, primal learning rate $\eta_p$, dual learning rate $\eta_d$, number of PDCL iterations $T$, number of primal steps per iteration $T_p$, constraint vector $\boldsymbol{\epsilon}$.
%   \STATE Initialize: $\mathcal{B} \leftarrow \emptyset$. 
  \STATE 
  $
  \boldsymbol{\theta}^{\star},\: \boldsymbol{\lambda}^{\star} \leftarrow  \text{PDCL}(\mathcal{L}, \eta_p, \eta_d, T, T_p, \boldsymbol{\epsilon}).
  $ \algcomment{Run the Primal-Dual Algorithm}
%   \STATE Compute embeddings of the labeled samples,  $$\{ f_{\boldsymbol{\phi}}(\boldsymbol{x}_i) \}_{i \in {\mathcal{N_L}}}.$$
  \STATE
  $
  \boldsymbol{\omega}^{\star} \leftarrow \argmin_{\boldsymbol{\omega}} \frac{1}{|\mathcal{N}_{\mathcal{L}}|} \sum_{i\in\mathcal{N}_{\mathcal{L}}} \left\|f_{\boldsymbol{\omega}}(f_{\boldsymbol{\phi}^{\star}}(\boldsymbol{x}_i)) - \lambda_i^{\star} \right\|^2.
  $ \algcomment{Train the dual regression head}
  \STATE 
  $
  j^{\star} \: \leftarrow  \: \arg \max_{j\in\mathcal{N}_{\mathcal{U}}} f_{\boldsymbol{\omega}^{\star}}(f_{\boldsymbol{\phi}^{\star}}(\boldsymbol{x}_j))
  $ \algcomment{Find sample with highest dual variable}
  \STATE {\bfseries Return:}{ $j^{\star}$. }
\end{algorithmic}
\end{algorithm}

\subsection{Connection to BADGE~\citep{BADGE}}

The BADGE method~\citep{BADGE} uses the gradient of the loss function with respect to the parameters of the last layer, denoted by $\boldsymbol{\theta}_L$, as a measure of informativeness, i.e.,
\begin{align}
\label{xoptbadge}
%\boldsymbol{x^{\star}} = \arg \max_{\boldsymbol{x} \in \mathcal{U}}
 \frac{\partial\ell (f_{\boldsymbol{\theta}}(\boldsymbol{x}), \hat{y}(\boldsymbol{x}))}{\partial \boldsymbol{\theta}_L} ,
\end{align}
where $\hat{y}(\boldsymbol{x})$ is the \emph{hypothetical} label of $\boldsymbol{x}$, defined as $\hat{y}(\boldsymbol{x})\coloneqq \arg \max_{y\in\mathcal{Y}} [f_{\boldsymbol{\theta}}(\boldsymbol{x})]_y$. In contrast, as discussed in Theorem~\ref{theo:obj_derivative}, to evaluate the informativeness of a given sample, ALLY observes
\begin{align}
\frac{ \partial P^{\star} }{ \partial \epsilon(\boldsymbol{x})} = \frac{ \partial P^{\star} }{ \partial \boldsymbol{\theta}} \: \frac{ \partial \boldsymbol{\theta}}{ \partial \epsilon(\boldsymbol{x})} =  \frac{ \partial \mathbb{E}_{(\boldsymbol{x}, y) \sim \mathfrak{D}}\left[\ell \left(f_{\boldsymbol{\theta^{\star}}}(\boldsymbol{x}), y\right)\right] }{ \partial \boldsymbol{\theta}} \: \frac{ \partial \boldsymbol{\theta}}{ \partial \epsilon(\boldsymbol{x})}.\label{usvbadge}
\end{align}

There are three main differences between these informativeness measures:
\begin{itemize}
    \item ALLY uses the derivative of the average optimal loss \emph{over the entire distribution}, whereas BADGE only considers the point-wise derivative. In fact, most strategies evaluate their scoring function, or informativeness measure, on a single sample. Note that the term  $\frac{ \partial \mathbb{E}_{(\boldsymbol{x}, y) \sim \mathfrak{D}}\left[\ell \left(f_{\boldsymbol{\theta^{\star}}}(\boldsymbol{x}), y\right)\right] }{ \partial \boldsymbol{\theta}}$ is constant for all $\mathbf{x}$.
    \item ALLY observes the gradient with respect to \emph{all} model parameters, not only the ones in the last layer.
    \item Aside from the derivative of the loss with respect to the model parameters, ALLY also considers an additional term $\frac{ \partial \boldsymbol{\theta}}{ \partial \epsilon(\boldsymbol{x})}$, which models how the results of the optimization (i.e., model parameters) change when the constraint function is perturbed.
\end{itemize} 

\subsection{Interpreting the Informativeness Score of ALLY}
\label{sec:generation}

\iffalse
\begin{figure}[b]
 %\centering
 \quad \quad 
   \begin{minipage}[c]{0.35\textwidth}
    \caption{Sample generation by maximization of predicted dual variables. The top row shows the initial images from MNIST, to which the predictor associates a low dual variable. Rows 2-5 display the images resulting from subsequent iterations of gradient ascent. As the predicted dual variable increases, patterns corresponding to other classes appear, increasing the uncertainty on the true image label.}       \label{fig:sample_generation}
  \end{minipage} \hfill \quad \quad \quad
  \begin{minipage}[c]{0.65\linewidth}
    \includegraphics[trim=1.5in 1.5in 1.2in 1.5in, clip, width = 0.65\linewidth]{images/synthesis/evol_lambda.pdf}
  \end{minipage}
\end{figure}
\fi
Our proposed framework allows us to leverage the trained backbone and dual regression head in a \emph{generative} manner to create novel samples that are most informative. Here, we focus on the MNIST dataset, and use the trained model to generate synthetic images with maximal associated dual variables. We begin by training a MLP %\footnote{The architecture of this MLP is described in section \ref{sec:experiments_setts}} 
on 10\% of the MNIST dataset. Then, similarly to~% the method described in
\citep{goodfellow_adv}, we perform gradient ascent on images that are initially considered uninformative, so as to maximize their predicted dual variables. The progression of the resulting images is shown in Figure \ref{fig:sample_generation}. 

\begin{figure}[b]
    \centering
    \includegraphics[trim=1.5in 1.5in 1.2in 1.5in, clip, width = 0.4\linewidth]{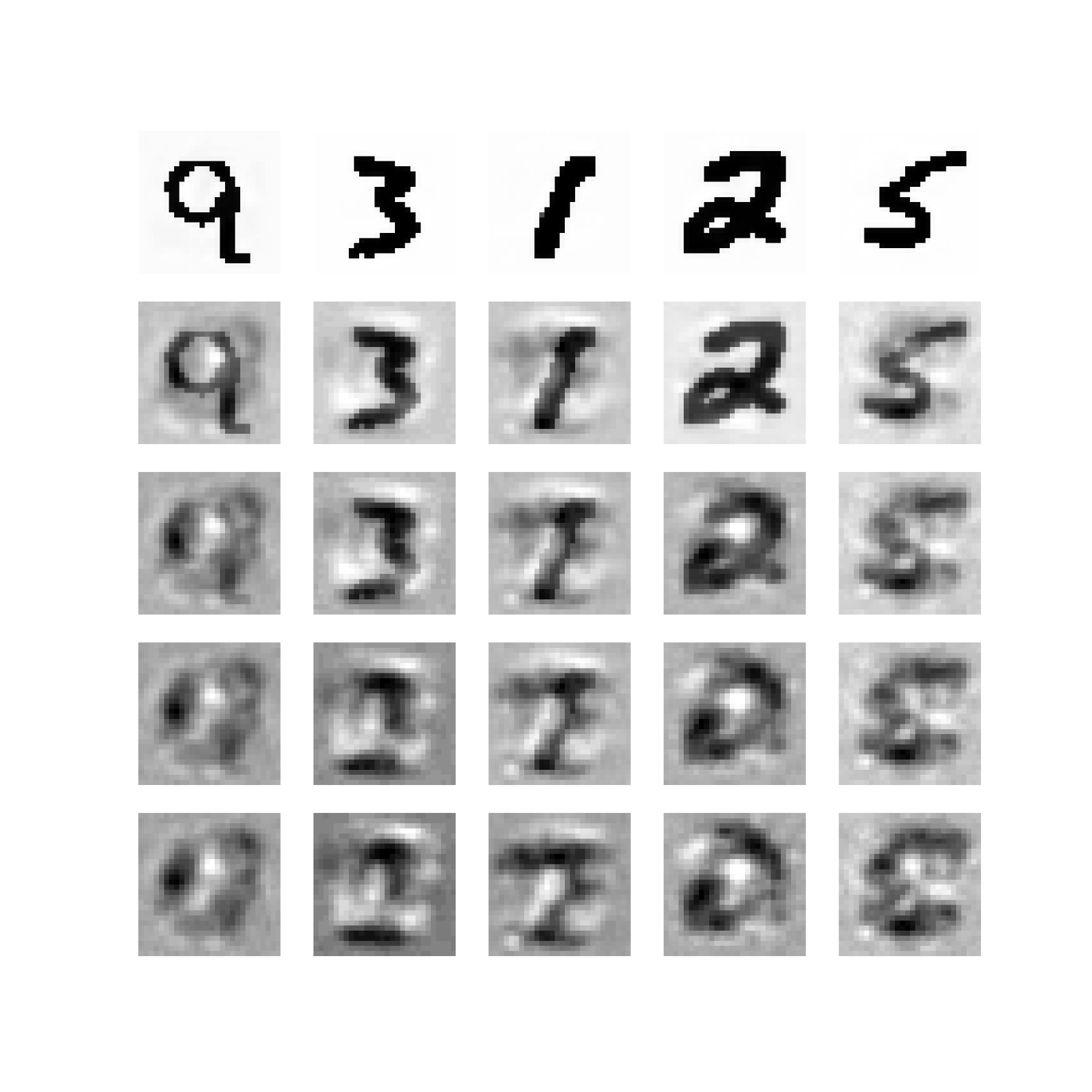}
    \caption{Sample generation by maximization of predicted dual variables. The top row shows the initial images from MNIST, to which the predictor associates a low dual variable. Rows 2-5 display the images resulting from subsequent iterations of gradient ascent.}% As the predicted dual variable increases, patterns corresponding to other classes appear, increasing the uncertainty on the true image label.}  
    \label{fig:sample_generation}
\end{figure}

As the predicted dual variable increases, patterns corresponding to other digits appear, increasing the uncertainty on the true label of the image. For instance, the third column of Figure \ref{fig:sample_generation} shows a handwritten `1' that is progressively transformed into a blurred superposition of `7,' `2' and '1'. Images in the last row of Figure \ref{fig:sample_generation} can be interpreted as lying in the tails of the distribution $\mathfrak{D}$, or close to the decision boundary of the end-to-end model $f_{\boldsymbol{\theta}}$. This also suggests that informative samples and outliers (such as mislabeled samples) may be hard to distinguish. Recent empirical findings indicate that many active learning algorithms consistently prefer to acquire samples that traditional models fail to learn \citep{outliers}. Thus, modifying ALLY in order to avoid sampling these so-called \emph{collective outliers} (e.g., by setting an upper bound on the dual variable associated to the queries) would be desirable in datasets that are not highly curated.

\iffalse
This experiment shows that the trained model can shed light on the informativeness measure induced by ALLY. In addition, it demonstrates that informative samples and outliers (such as mislabeled samples) may be hard to distinguish. Recent empirical findings suggest that many active learning algorithms consistently prefer to acquire samples that traditional models fail to learn \citep{outliers}. Modifying ALLY in order to avoid sampling these so-called \emph{collective outliers} (e.g., by setting an upper bound on the dual variable associated to the queries) is a promising research direction that we leave for future work.
\fi

\vspace{-.15in}
\section{Experimental Evaluation}
\label{sec:experiments}

\subsection{Settings}\label{sec:experiments_setts}
We consider four image classification tasks and one biomedical, non-image regression task. In the classification setting, we use standard datasets that commonly appear in the active learning literature, namely STL-10 \citep{stl-10}, CIFAR-10 \citep{cifar10}, SVHN \citep{svhn} and MNIST \citep{mnist}. Lacking an established benchmark regression dataset for active learning, we evaluate ALLY on the Parkinsons Telemonitoring dataset (PTD) \citep{parkinsondataset}. In this regression task, the goal is to predict UPDRS (Unified Parkinson's Disease Rating Scale) scores from dysphonia measurements such as variation in fundamental frequency. Since measurements in this dataset are the result of a \textit{costly} clinical trial that requires expert knowledge, this task is a prime example in which active learning might be essential. 

In all experiments, the initial labeled set $\mathcal{L}_0$ consists of 200 randomly drawn samples, and the budget is set to either $b=200$ or $b=1000$. We use a ResNet-18 architecture~\citep{resnet} with an embedding size of 128. In the case of MNIST and PTD, which are simpler tasks, we use a multi-layer perceptron (MLP) with two hidden layers, each with 256 neurons and rectified linear unit (ReLU) activation, leading to an embedding size of 256. The dual regression head $f_{\omega}$ is a MLP with 5 hidden layers, ReLU activations and batch normalization. 

% \subsection{Batch Diversity}

% Some prominent diversity-promoting approaches cluster the unlabeled set and then select informative points from each cluster \citep{div_kmeans1, div_kmeans2, div_kmeans3}. In \citep{div_approx}, Monte Carlo sampling is used to simulate sequences of length $b$ of the sequential algorithm, and then a \textit{best-matching} combination of the sequences is used to build a batch. A simpler approach is to select the $b$ most informative points after a stochastic perturbation of informativeness scores \citep{stoch_db}.
As done in \citep{div_kmeans1, div_kmeans2}, to ensure diversity in the batch, we cluster the embeddings of the unlabeled samples, i.e., $\{f_{\boldsymbol{\phi}^{\star}}(\boldsymbol{x}_j)\}_{j \in {\mathcal{N_U}}}$, using the $k$-MEANS clustering algorithm \citep{kmeans}, where $k \leq b$ is a hyperparameter. We then select the samples with the highest associated dual variables from each cluster, while maintaining equity among the number of samples per cluster. As shown in Appendix~\ref{app:ablation_clust}, the performance of ALLY improves with an increased number of clusters $k$, since it leads to a more diverse batch of selected samples. Therefore, in all our experiments, we set $k=b$. %Figure~\ref{fig:ally} illustrates an overview of ALLY.
%\vspace{-.1in}
\begin{figure}[b]
     \centering
     \begin{subfigure}[b]{\textwidth}
         \centering
         \includegraphics[width=.8\textwidth]{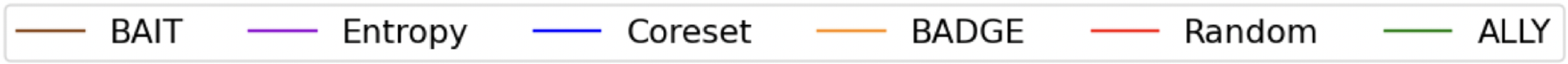}
     \end{subfigure}
     \begin{subfigure}[b]{0.24\textwidth}
         \centering
         \includegraphics[width=.99\textwidth]{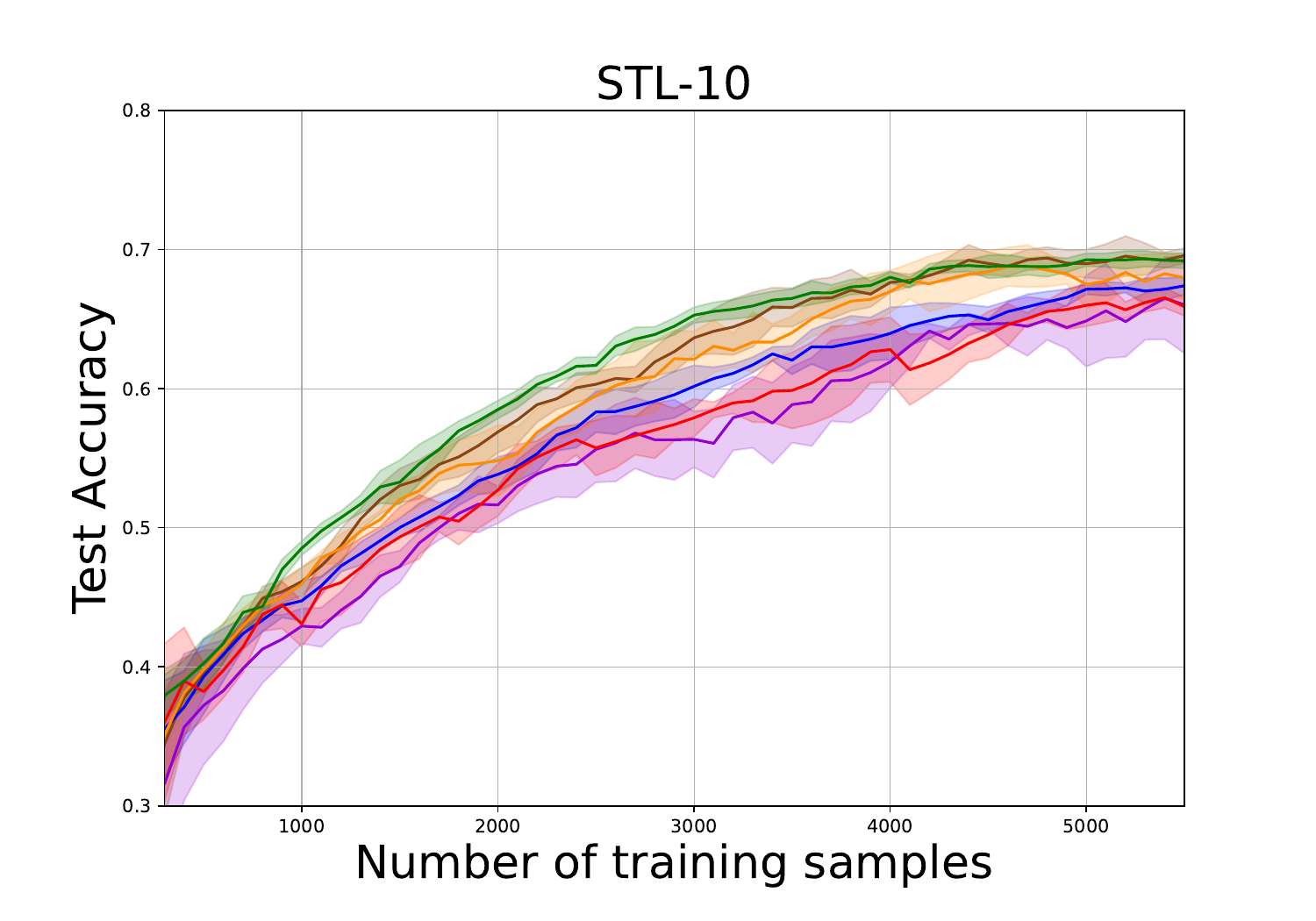}
     \end{subfigure}
     \begin{subfigure}[b]{0.24\textwidth}
         \centering
         \includegraphics[width=.99\textwidth]{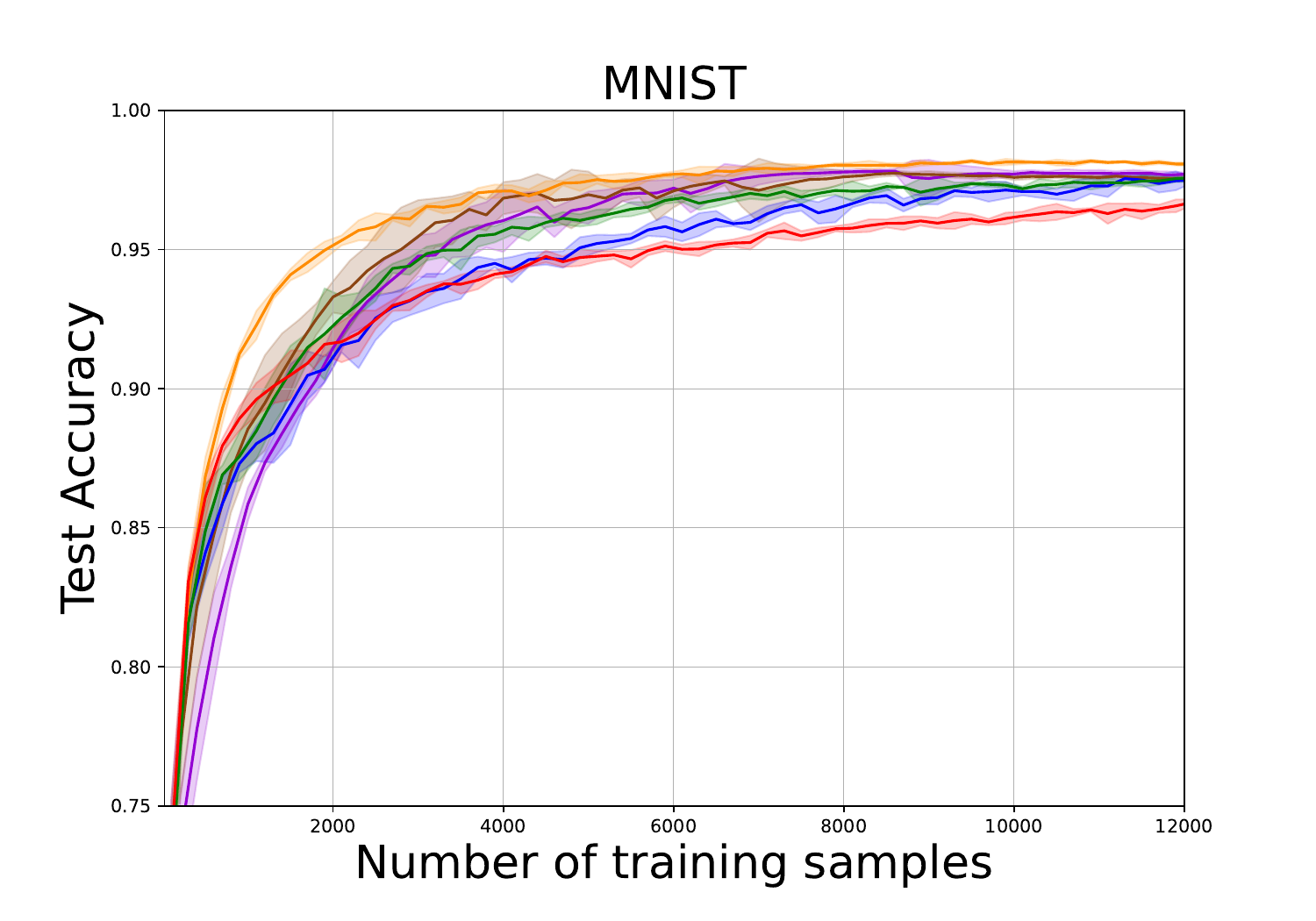}
     \end{subfigure}
     \begin{subfigure}[b]{0.24\textwidth}
         \centering
         \includegraphics[width=.99\textwidth]{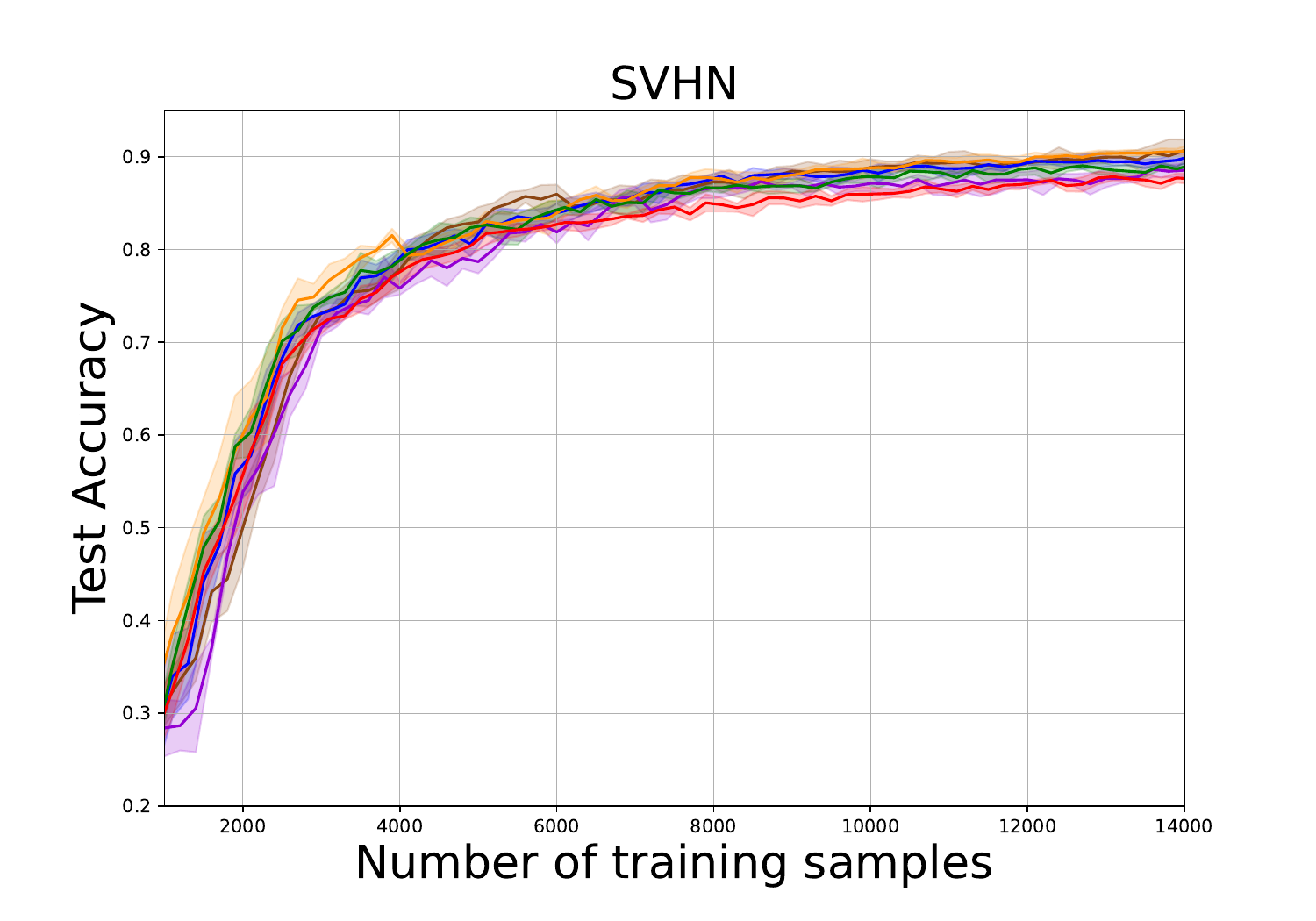}
     \end{subfigure}
         \begin{subfigure}[b]{0.24\textwidth}
         \centering
         \includegraphics[width=.99\textwidth]{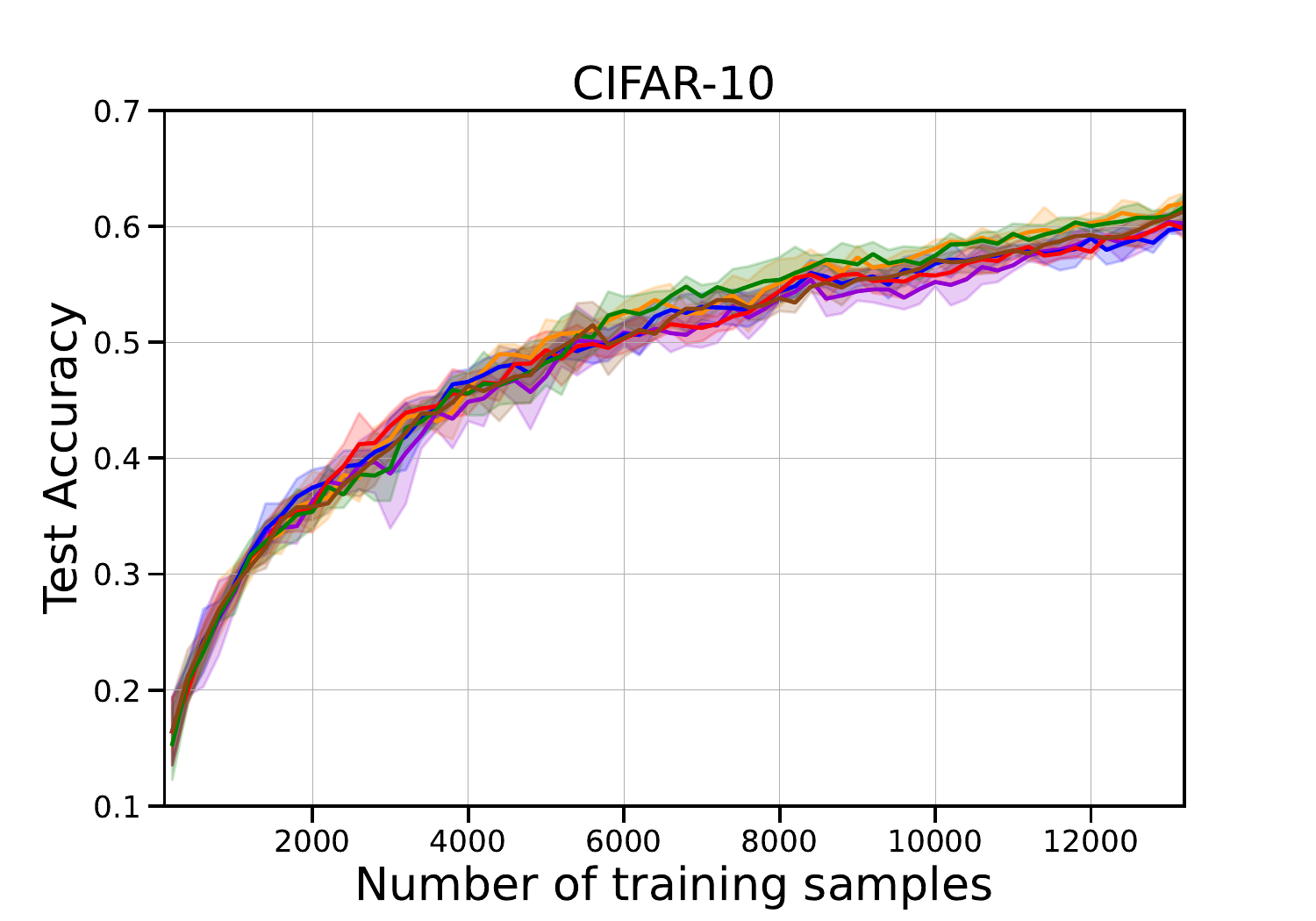}
\end{subfigure}
     \begin{subfigure}[b]{0.24\textwidth}
         \centering
         \includegraphics[width=.99\textwidth]{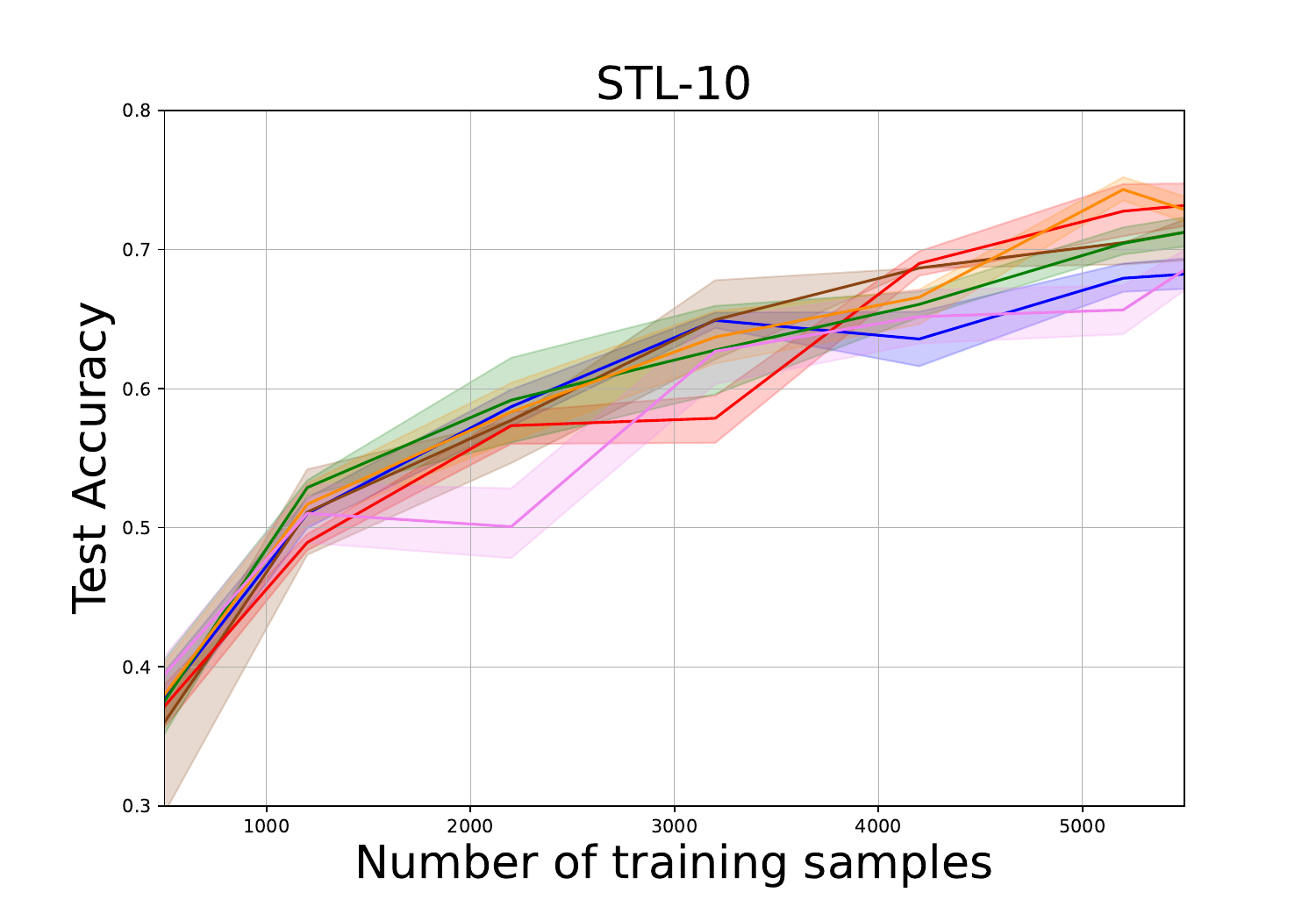}
     \end{subfigure}
     \begin{subfigure}[b]{0.24\textwidth}
         \centering
         \includegraphics[width=.99\textwidth]{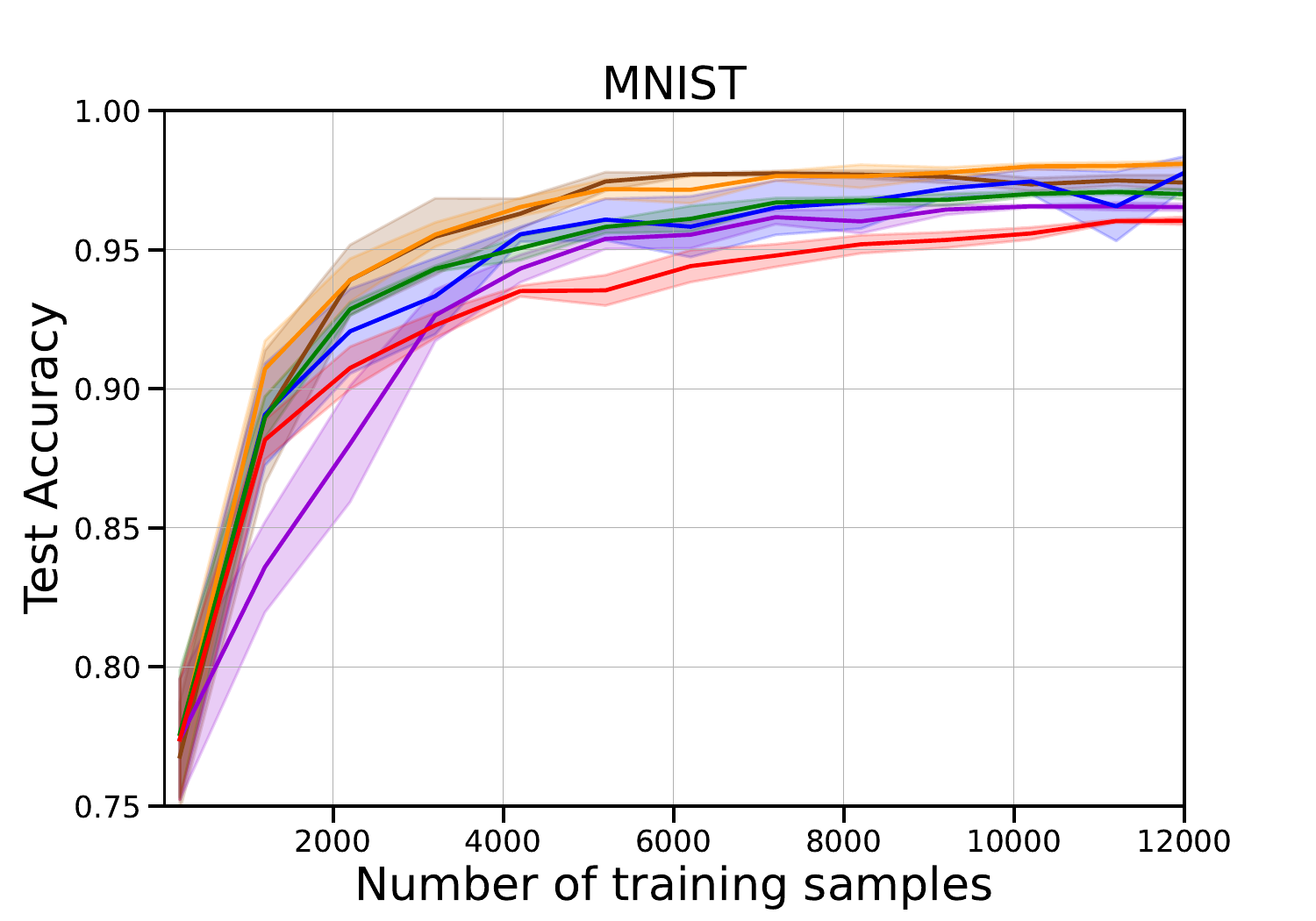}
     \end{subfigure}
     \begin{subfigure}[b]{0.24\textwidth}
         \centering
         \includegraphics[width=.99\textwidth]{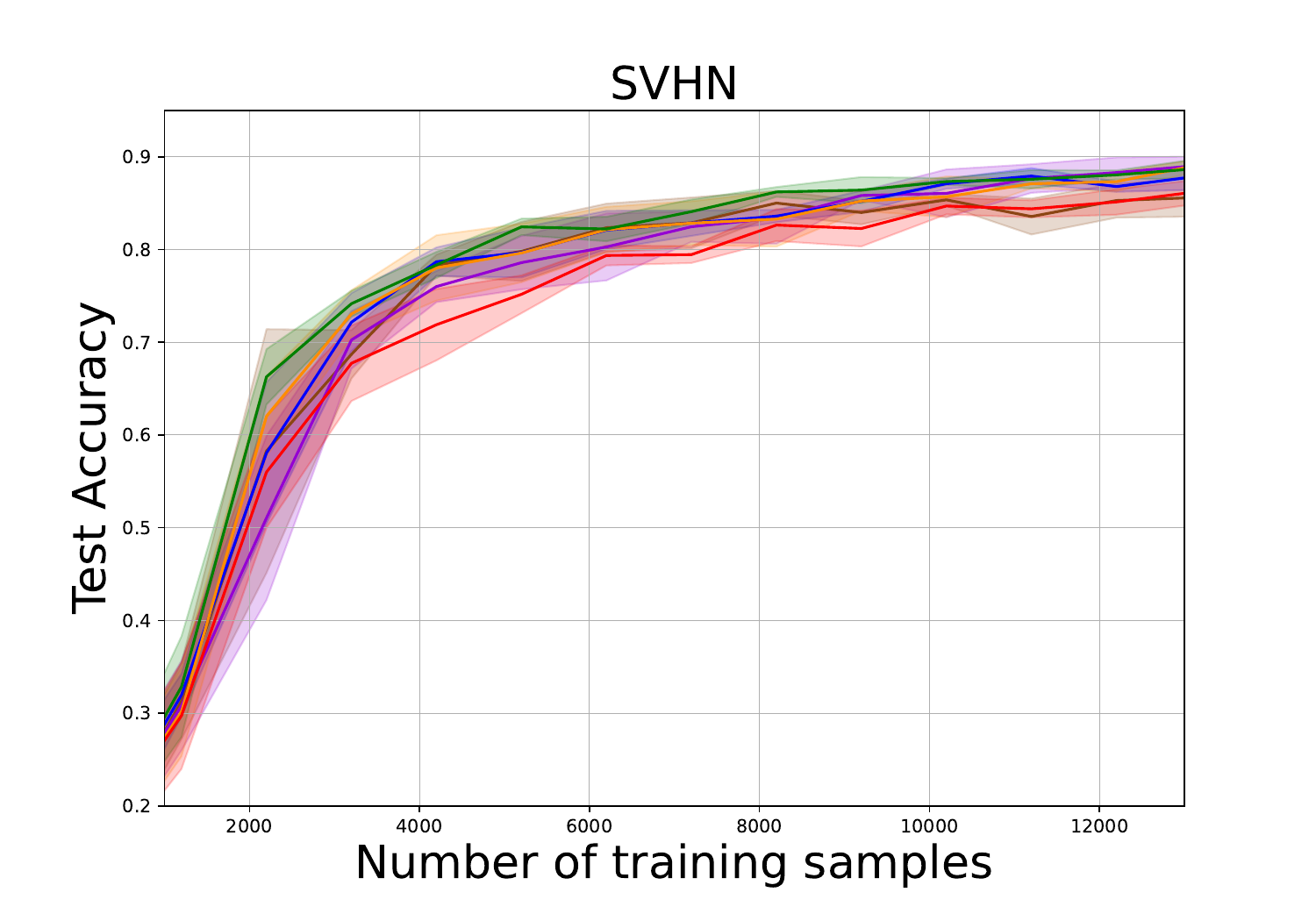}
     \end{subfigure}
         \begin{subfigure}[b]{0.24\textwidth}
         \centering
         \includegraphics[width=.99\textwidth]{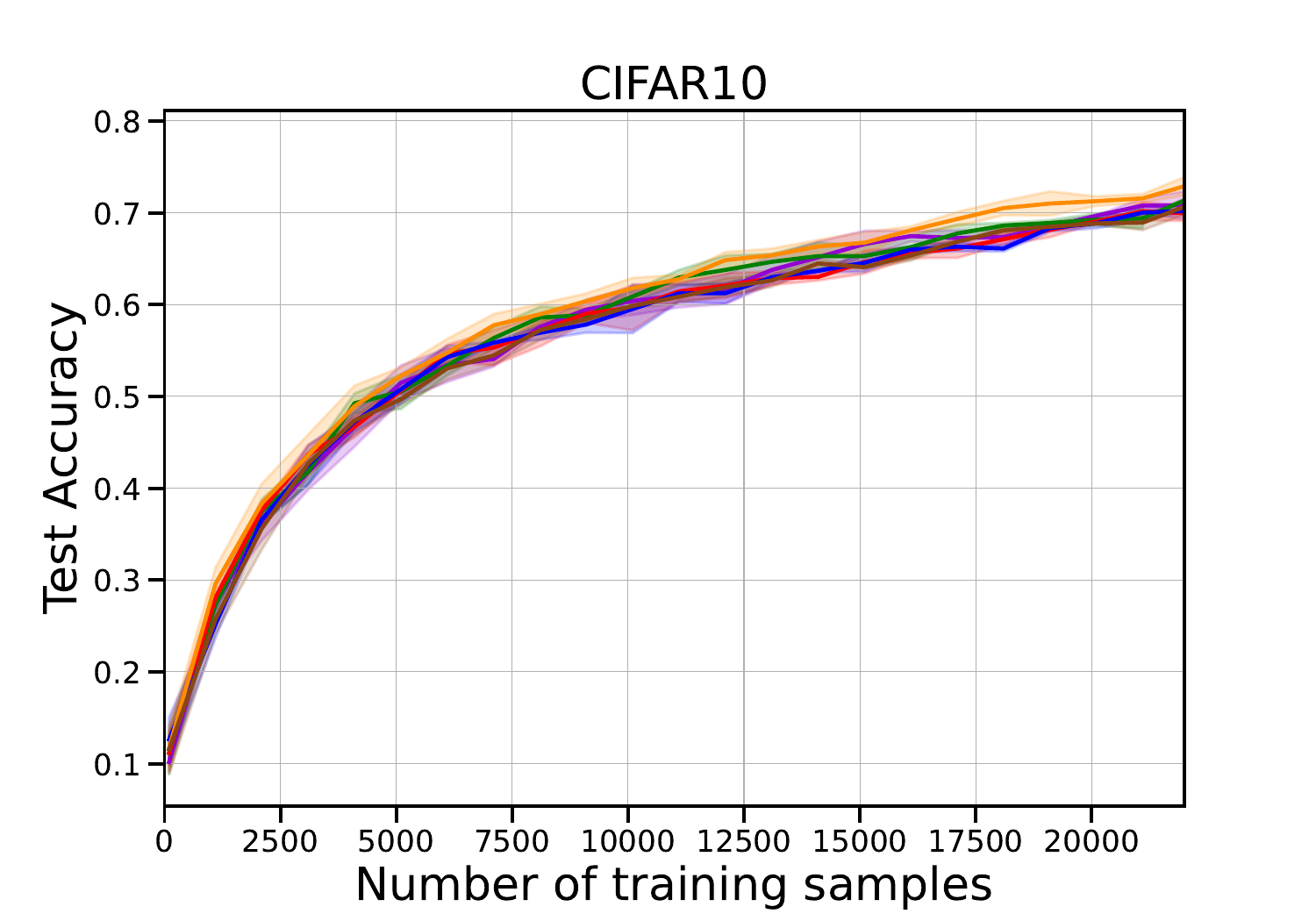}
     \end{subfigure}

\caption{Accuracy in the test set as a function of the number of training samples in four classification settings and two budgets, $200$ (top) and $1000$ (bottom). Solid curves represent the mean across five different random seeds, while shaded regions correspond to the standard deviation.}
\label{fig:classification}
\end{figure}

Regarding the role of the secondary loss, we opted for a generic formulation as the performed sensitivity analysis in Theorem~\ref{theo:obj_derivative} holds for various choices of $\ell'$. However, in all our experiments, we set $\ell'(\cdot, \cdot) = \ell(\cdot, \cdot)$. We believe that using unsupervised or self-supervised losses for $\ell'$ is a promising research direction, and we leave it for future work.

We compare our algorithm with Entropy sampling, BADGE, Coreset and BAIT. While Entropy sampling focuses purely on uncertainty, BADGE and BAIT balance both diversity and informativeness with different approaches. Coreset differs from the previous methods in that it focuses on batch representativeness by framing active learning as a coreset selection problem. It has been observed that, in some scenarios, several active learning methods fail to consistently outperform Random Sampling \citep{outliers, simil_coreset, disc_active_learning}. We thus include it as one of the five baselines. As explained in section \ref{proposed_app}, ALLY uses a primal-dual approach to learn $f(\mathbf{x}; \mathcal{L} \cup \mathcal{U)}) $, while all other baselines are optimized using stochastic gradient descent (ADAM). We adopt the PyTorch \citep{pytorch} implementation of the baselines from \citep{libact}. % More details on the experimental setting can be found in Appendix E.

\subsection{Classification}

The experiments on STL-10, CIFAR-10, SVHN and MNIST are all 10-class, image classification tasks. We use the cross-entropy loss for both $\ell(\cdot, \cdot)$ and $\ell'(\cdot, \cdot)$ and set $\epsilon(\boldsymbol{x}) =  0.2, \, \forall \boldsymbol{x}$. As shown in Figure \ref{fig:classification}, ALLY exhibits top-2 performance with respect to other baselines in STL-10 (budget 200), SVHN (budget 1000) and CIFAR-10 (budget 200). In these three datasets, the improvement in the number of samples needed by ALLY to achieve 97\% of the final accuracy, in comparison with the best baseline, is 9\%, 8\% and 2\%, respectively. In MNIST, however, BADGE and BAIT consistently outperform ALLY. This may be due to the fact that the MLP, being less expressive, yields embeddings of lower quality, hindering the prediction of dual variables.

% sample complexity gains of  STL-10, CIFAR-10, SVHN to be added here.

% Table ... shows the sample complexity results for STL-10, CIFAR-10, SVHN, where we identify the number of labeled samples needed to achieve 95\% of the final accuracy for each method.

\subsection{Regression}

We use mean-squared error for both $\ell(\cdot, y)$ and $\ell'(\cdot, y)$ and set $\epsilon(\boldsymbol{x}) = 0.1, \, \forall \boldsymbol{x}$. As seen in Figure \ref{fig:regression}, ALLY outperforms both Random and Coreset in this regression task, the gap being larger at the beginning of the learning curve. Note that BADGE and Entropy are not applicable, since they are limited to classification scenarios. %\footnote{Although BAIT can be used for regression, the implementation code for the regression version of BAIT is not publicly available.}

\vspace{-.1in}
\begin{figure}[h]
  \centering
  \includegraphics[width=.32\textwidth]{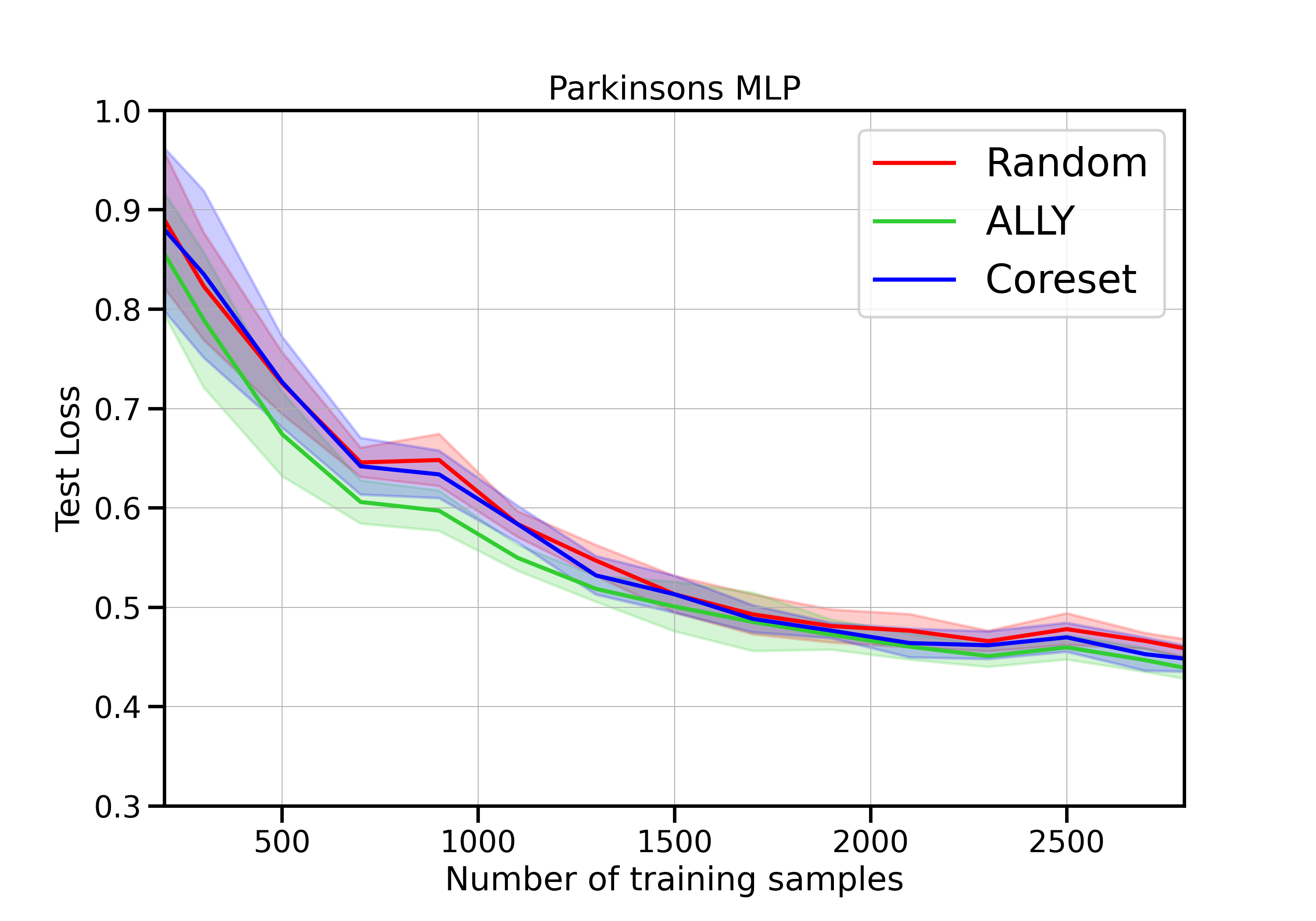}
  \captionof{figure}{Mean-squared error in the test set as a function of the number of training samples in the Parkinson's telemonitoring dataset.}
  \label{fig:regression}
\end{figure}

% \begin{figure}[H]
%  \centering
%     \includegraphics[width=.45\textwidth]{images/regression/parkinsonmnist2.png}
%     \caption{Mean-squared error in the test set as a function of the number of training samples in the Parkinson's telemonitoring dataset.}
%     \label{fig:regression}
% \end{figure}

% \iffalse
% \begin{figure}[H]
%  %\centering
%  \quad \quad 
%   \begin{minipage}[c]{0.35\textwidth}
%     \caption{Mean-squared error in the test set as a function of the number of training samples in the Parkinson's telemonitoring dataset.}
%     \label{fig:regression}
%   \end{minipage} \hfill \quad \quad \quad
%   \begin{minipage}[c]{0.65\textwidth}
%     \includegraphics[width=.65\textwidth]{images/regression/parkinsonmnist2.png}
%   \end{minipage}
% \end{figure}
% \fi

\vspace{-.1in}
\subsection{Ablation on the Constraint Tightness}
\begin{figure}[b]
     \centering
     \begin{subfigure}[b]{\textwidth}
         \centering
         \includegraphics[width=.3\textwidth]{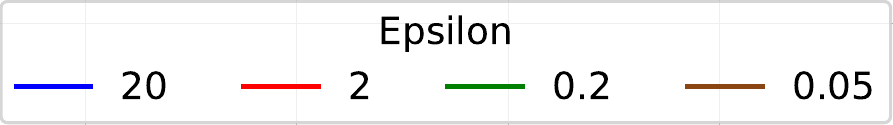}
     \end{subfigure}
     \begin{subfigure}[b]{0.3\textwidth}
         \centering
         \includegraphics[width=.99\textwidth]{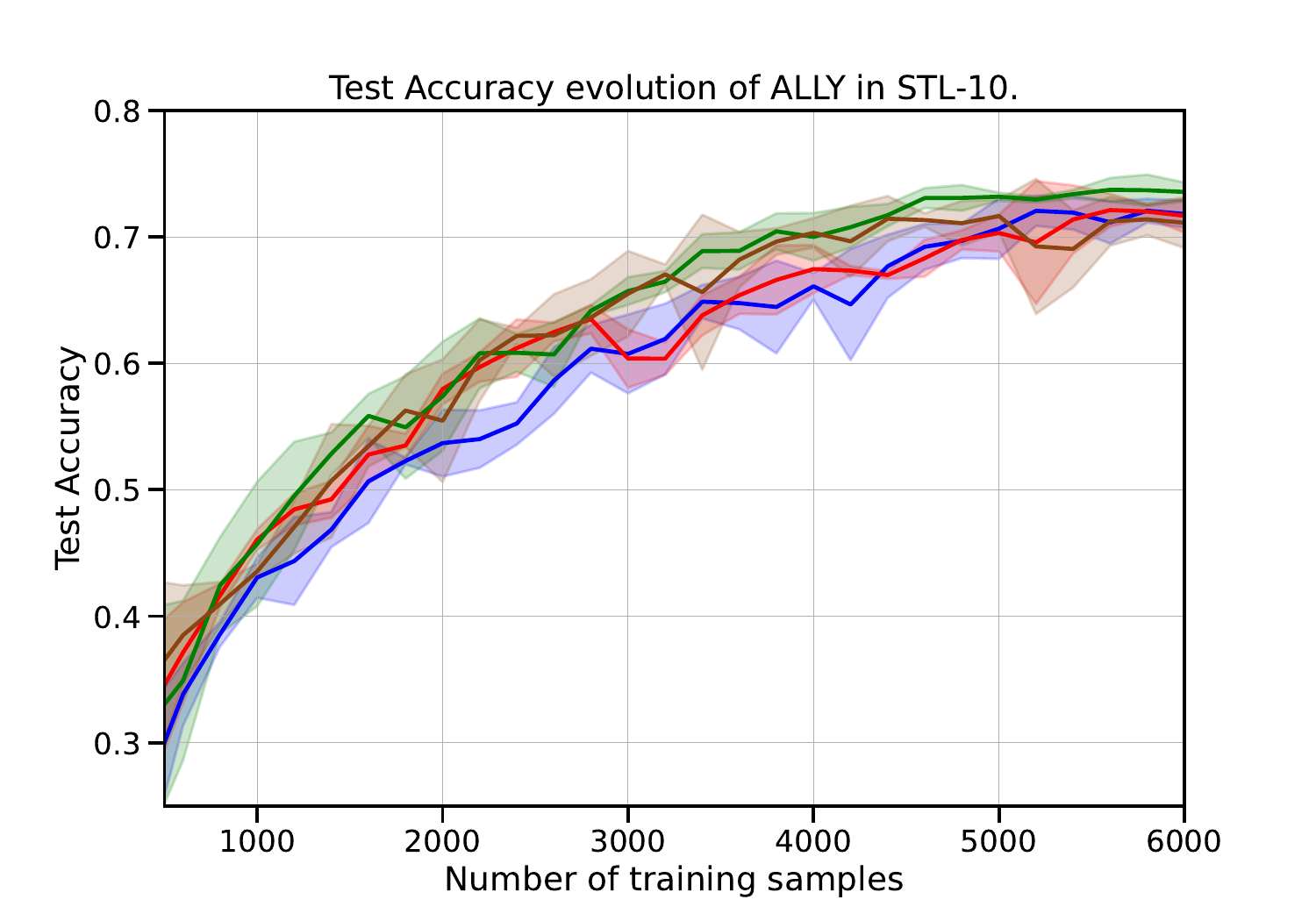}
     \end{subfigure}
     \begin{subfigure}[b]{0.3\textwidth}
         \centering
         \includegraphics[width=.99\textwidth]{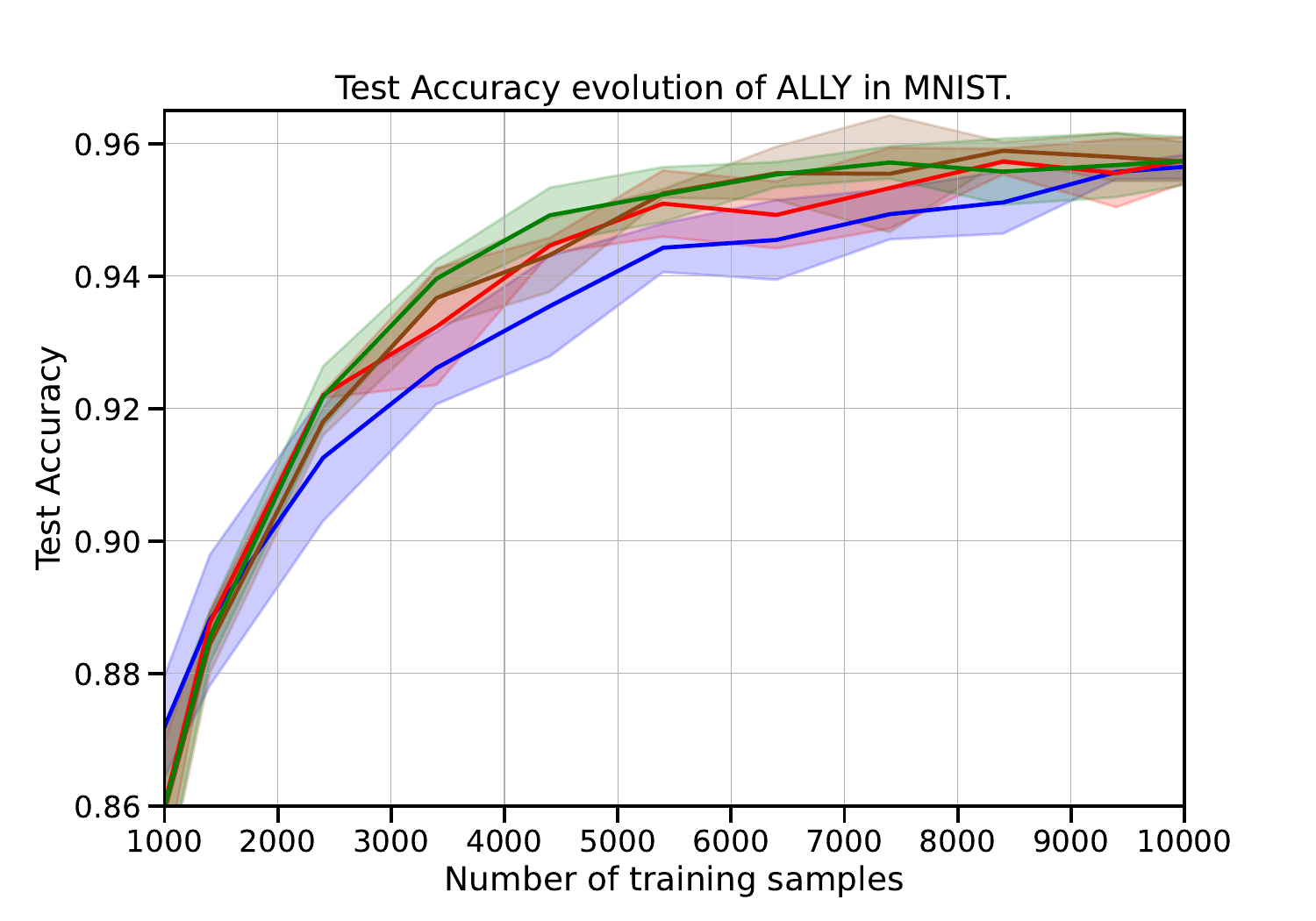}
     \end{subfigure}

\caption{Ablation on the value of the constraint tightness $\epsilon$ on the performance of ALLY in STL-10 and MNIST.}
\label{fig:ablation}
\end{figure}
Figure~\ref{fig:ablation} illustrates the role $\epsilon$ plays in the optimization procedure. For extremely large values of epsilon (e.g., $20$ nats), the constraint slacks become negative for all samples, and thus all dual variables become zero, making them uninformative (analogous to an unconstrained problem). Large values (e.g., $2$ nats) can potentially be used to detect outliers. Our ablations suggest that values in the range $[1.05p_u \, , \, 1.25p_u]$, where $p_u$ is the average loss observed when training the model without constraints, work well in practice. The sensitivity of the method with respect to the constraint level is impacted both by the task/dataset at hand and the diversity technique used. It is reasonable to expect K-means to be less sensitive to the constraint level than stochastic perturbations of the informativeness scores.

\subsection{Dataset Redundancy}

We study the performance of ALLY and Entropy Sampling with varying levels of dataset redundancy in STL-10. We define the level of dataset redundancy as the number of copies of each sample present in the training set.  We then analyze the evolution of test accuracy over several rounds of active learning with a query budget of 200.

\begin{figure}[H]
     \centering
     \begin{subfigure}[b]{\textwidth}
         \centering
         \includegraphics[width=.32\textwidth]{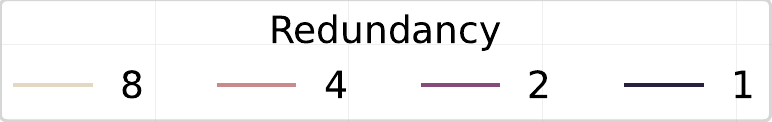}
     \end{subfigure}
     \begin{subfigure}[b]{0.32\textwidth}
         \centering
         \includegraphics[width=.99\textwidth]{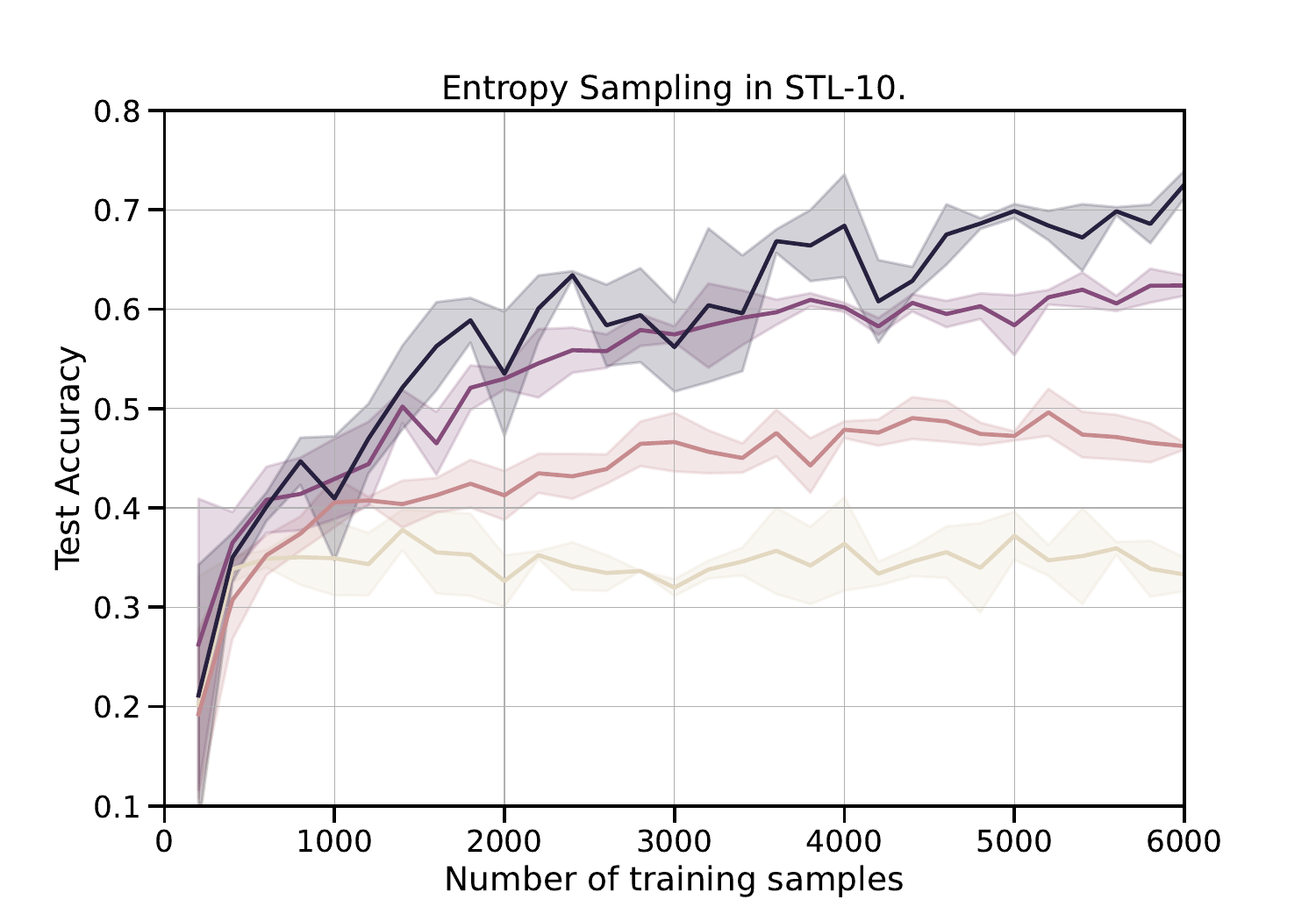}
     \end{subfigure}
     \begin{subfigure}[b]{0.32\textwidth}
         \centering
         \includegraphics[width=.99\textwidth]{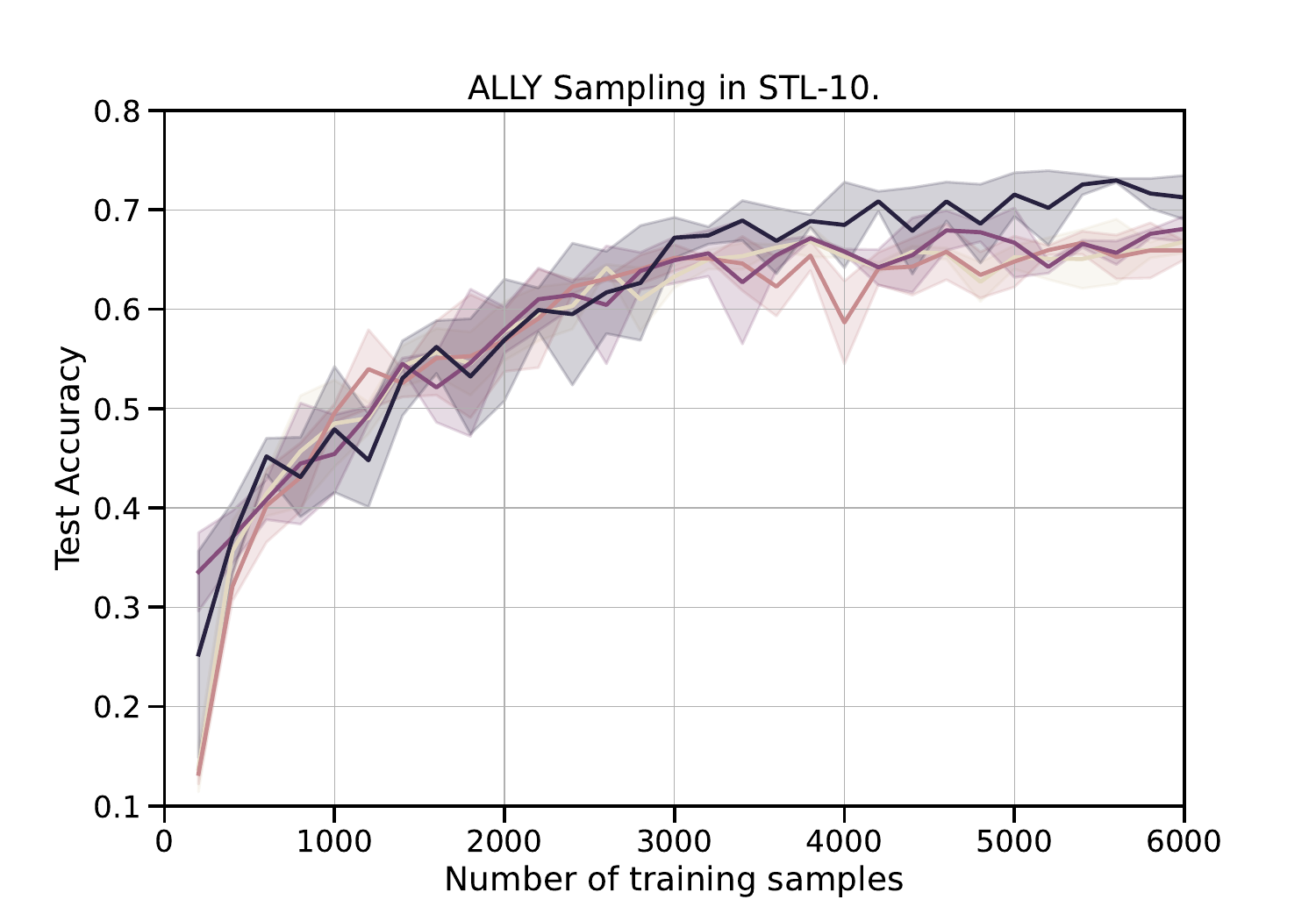}
     \end{subfigure}
    \begin{subfigure}[b]{0.32\textwidth}
         \centering
         \includegraphics[width=.99\textwidth]{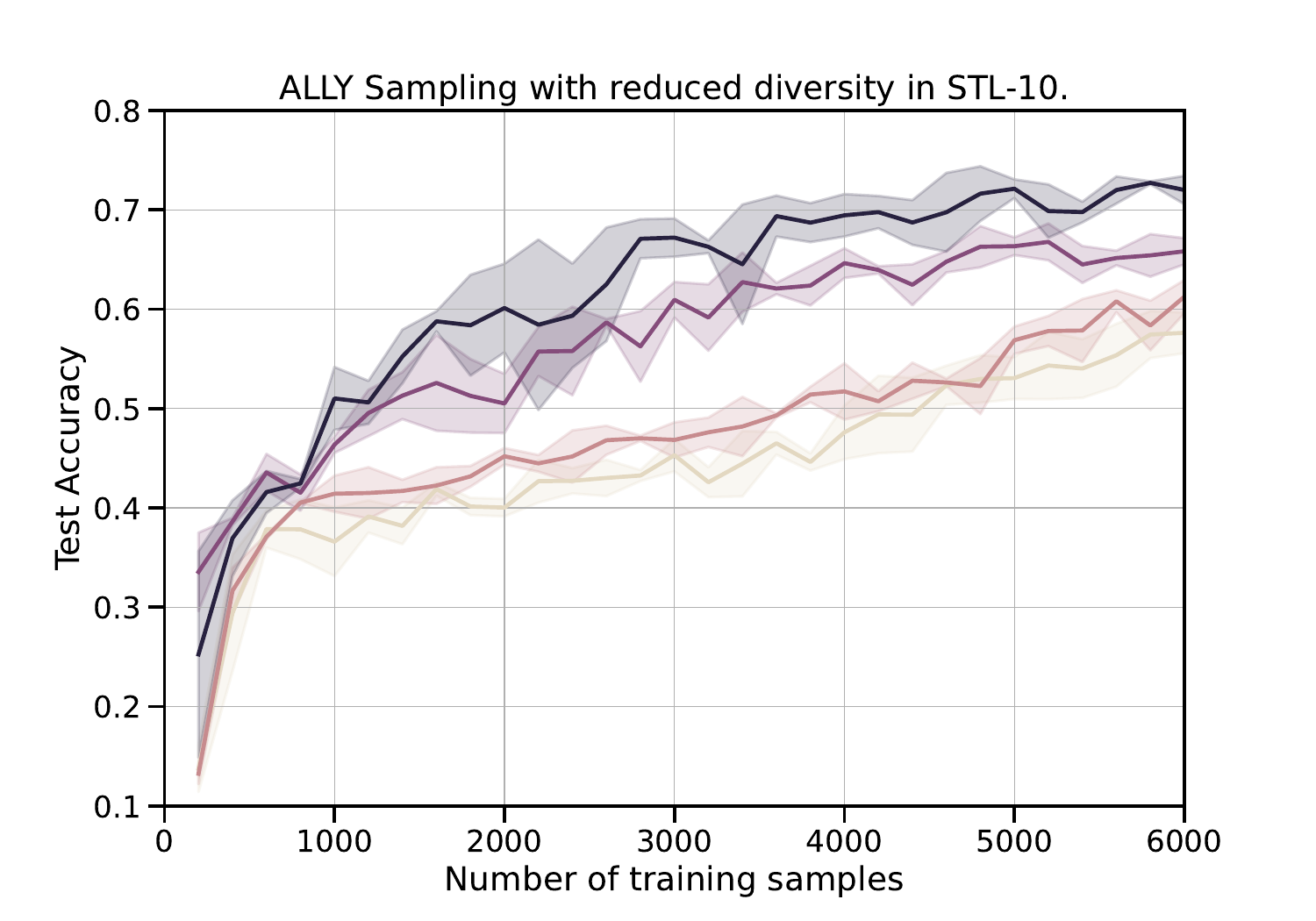}
     \end{subfigure}

\caption{Impact of dataset redundancy on the performance of Entropy Sampling (left), ALLY (center) and low-diversity ALLY (right).}
\label{fig:redundancy}
\end{figure}
\vspace{-.3in}
We observe that a strategy solely based on informativeness (e.g., Entropy Sampling) is very sensitive to this type of dataset redundancy. This can be attributed to the fact that, at each round, copies of the most informative samples are queried simultaneously, leading to low batch diversity. Moreover, ALLY appears to be more robust than Entropy Sampling to this type of dataset corruption. This is not surprising, since only one sample is queried from each cluster, avoiding batch information overlap. However, when increasing the number of samples queried from each batch, the performance of ALLY is degraded, resembling that of Entropy Sampling. This experiment suggests that the impact of dataset redundancy on active learning strategies is more linked to the diversity technique used than the informativeness measure. 
Whether this behaviour generalizes to other types of redundancy is left for future work.

\vspace{-.05in}
\section{Concluding Remarks}

We presented ALLY, a principled batch active learning method based on Lagrangian duality. Our method formulates the learning problem using constrained optimization and solves it in the Lagrangian dual domain via a primal-dual approach. We leverage the fact that the magnitude of the dual variables can be viewed as a measure of informativeness of the corresponding training sample, as it indicates the sensibility of the optimum value of the objective function with respect to a perturbation in the respective constraint. 

Following the completion of the primal-dual learning phase, we employ the learned sample representations, as well as their respective dual variables, to train a dual regression head. This predictor is used to estimate the dual variables associated to unlabeled samples. The resulting informativeness measure is compatible with several batch diversity promoting techniques. Using the k-MEANS algorithm, we demonstrated that this principled method exhibits competitive performance in several classification and regression experiments. We also showed that under certain conditions, the impact of the distribution shift induced by active sampling is small.
%We then promote diversity in the batch by clustering the unlabeled sample embeddings selecting the samples with the highest estimated dual variables from each cluster. 

%The image synthesis experiment shows that the trained model can shed light on the informativeness measure induced by ALLY. It also suggests a that informative samples and outliers (such as mislabeled samples) may be hard to distinguish. Recent empirical findings suggest that many active learning algorithms consistently prefer to acquire samples that traditional models fail to learn \citep{outliers}. Modifying ALLY in order to avoid sampling these so-called \emph{collective outliers} (e.g., by setting an upper bound on the dual variable associated to the queried samples) is a promising research direction that we leave for future work. 
In our experiments, we have set the secondary loss to be identical to the primary supervised loss (i.e., cross-entropy loss for classification, and mean-squared error for regression). However, %one can think of other secondary losses that are different than the primary loss.
evaluating the performance of ALLY under  alternative \emph{unsupervised} or \emph{self-supervised} secondary losses is a promising future direction. Finally, it would be interesting to evaluate the performance of ALLY under different diversity measures, comparing their computational burden. % Using the dual regression head to create novel and informative samples that can be used to further fine-tune the model is another avenue of research. Finally, 
%Moreover, other methods for clustering and selecting a diverse batch of unlabeled samples based on their embeddings and associated dual variables are worth investigating. 

\section{Acknowldegments}
Supported by NSF-Simons MoDL, Award \#2031985, NSF AI Institutes program, Award \#2112665, and NSF HDR TRipods Award \#1934960.

\newpage

\bibliographystyle{unsrtnat}
\bibliography{references} 

%\section*{References}

\newpage
\newpage
\include{appendix}

\end{document}

%% file: appendix.tex
\appendix

\section{Strong Duality of the CSL Problem}
\label{appx:duality}
\textbf{Assumption 1.}\label{ass:1} The losses $\ell(\cdot, y)$ and $\ell'(\cdot, y)$, are convex functions for all $y \in \mathcal{Y}$. %The loss $\ell$ is additionally strongly convex. 

In Assumption 1, the convexity of the losses is taken with respect to the model's output, and not the model parameters. The cross-entropy loss, commonly used in classification with a softmax layer satisfies strong convexity when considering a probability simplex \cite{crossent_lipschitz}. Typical losses for regression (e.g., mean-squared error, L1 loss) also satisfy this assumption. 

\textbf{Assumption 2.} The hypothesis class $\mathcal{F}$ is convex. 

To obtain a convex hypothesis class, as required by Assumption 2, it suffices to take the convex hull of the function class originally considered. 

%such that $\mathcal{F} \subseteq \hat{\mathcal{F}}$.  and a constant $\nu>0$ such that for each $f \in \hat{\mathcal{F}}$ there exists $f_{\theta} \in \mathcal{F}$ for which $\sup _{\boldsymbol{x} \in \mathcal{X}}\left|f_{\boldsymbol{\theta}}(\boldsymbol{x})-f(\boldsymbol{x})\right| \leq \nu$.

\textbf{Assumption 3.} There exists $f \in \mathcal{F}$ strictly feasible for \eqref{eq:cl_formulation} (i.e., $\ell' \left(f(\boldsymbol{x}), y\right) < \epsilon(\boldsymbol{x}), \quad \mathfrak{D} \text {--a.e.}$)

Assumption 3 guarantees that the problem \eqref{eq:cl_formulation} is feasible and that its dual is well-posed.

%There exists $\boldsymbol{\theta}^{\prime} \in \mathbb{R}^{p}$ such that $f_{\boldsymbol{\theta}^{\prime}}$ is strictly feasible for \ref{eq:cl_formulation} with constraint $\epsilon(\boldsymbol(x)-M \nu$ and for each datasets $\mathcal{S}=\left\{\left(\boldsymbol{x}_{n_{i}}, y_{n_{i}}\right)\right\}_{i=0, \ldots, m+q}$ there exists a $\boldsymbol{\theta}^{\prime \prime}$ that is strictly feasible for \ref{eq:cl_formulation}.

\begin{proposition}
\label{prop_strong_duality}
Under Assumptions 1-3, \eqref{eq:cl_formulation} and \eqref{eq:D-CSL} are strongly dual, i.e., $P^* = D^*$.
\end{proposition}

% \textit{Proof.}
\begin{proof}
Note that assumptions 1 and 2 imply that \eqref{eq:cl_formulation} is a \emph{convex} program. Under the strict feasibility assumption, \eqref{eq:cl_formulation} satisfies the constraint qualification known as \emph{Slater's condition}, from which strongly duality follows \cite{bertsekasconvex, shapiroduality}.  
\end{proof}

%is related to the richness of the considered function class. the parameter $\nu$ can be decreased by increasing the capacity (e.g: by adding more layers or neurons in the case of a neural network). The class \mathcal{F} can denote the space of continuous functions, or a Reproducing Kernel Hilbert Space. To obtain a convex hypothesis class $\hat{\mathcal{F}}$, it suffices to take the convex hull of $\mathcal{F}$.

\section{Proof of Theorem~\ref{theo:obj_derivative} (Sensitivity of $P^{\star}$)}
\label{appx:theo_proof}
This result stems from a sensitivity analysis on the constraint of problem \eqref{eq:cl_formulation} and is well-known in the convex optimization literature. More general versions of this theorem are shown in \cite{shapirobook} (Section 4), \cite{ShapiroDifferentiability} or \cite{zencke}.

We start by viewing \eqref{eq:cl_formulation} as an optimization problem parameterized by the function $\epsilon(\boldsymbol{x})$:
\equationprefix{uCSL} % must be before subequation!
\begin{subequations}
\label{eq:unperturbed}
\renewcommand{\theequation}{}
\begin{alignat*}{2}
    P^{\star}(\epsilon(\boldsymbol{x})) = &\min_{f \in \mathcal{F}}  &\quad& \mathbb{E}_{\mathfrak{D}}\left[\ell \left(f(\boldsymbol{x}), y\right)\right] \\
    &~~\text{s.t.} && \ell' \left(f(\boldsymbol{x}), y\right) \leq \epsilon(\boldsymbol{x}), \quad \mathfrak{D}_\mathbf{x} \text {--a.e. }
\end{alignat*}
\end{subequations}
\renewcommand{\theequation}{\arabic{equation}}%
\equationprefix{}% no prefix

Define the Lagrangian $L(f, \lambda(\boldsymbol{x});\epsilon(\boldsymbol{x}))$ as $$L(f, \lambda(\boldsymbol{x});\epsilon(\boldsymbol{x})) =  \mathbb{E}_{(\boldsymbol{x}, y) \sim \mathfrak{D}} \: \Big[ \ell (f(\boldsymbol{x}), y)  + \lambda(\boldsymbol{x})( \ell'(f(\boldsymbol{x}), y)  - \epsilon(\boldsymbol{x}) )\Big],$$ where the dependence on $\epsilon(\boldsymbol{x})$ is explicitly shown. Then, following the definition of $P^\star(\epsilon(\boldsymbol{x}))$ and using strong duality, we have $$P^\star(\epsilon(\boldsymbol{x}))=\min_f L(f, \lambda^{\star}(\boldsymbol{x};\epsilon(\boldsymbol{x}));\epsilon(\boldsymbol{x})) \leq L(f, \lambda^{\star}(\boldsymbol{x};\epsilon(\boldsymbol{x}));\epsilon(\boldsymbol{x}))$$ with the inequality being true for any function $f$, and where the dependence of $\lambda^{\star}$ on $\epsilon(\boldsymbol{x})$ is also explicitly shown. Now, consider an arbitrary function $\epsilon'(\boldsymbol{x})$ and the respective primal function $f^{\star}(\cdot;\epsilon'(\boldsymbol{x}))$ which minimizes its corresponding Lagrangian. Plugging $f^{\star}(\cdot;\epsilon'(\boldsymbol{x}))$ into the above inequality, we have
\begin{align*}
P^\star(\epsilon(\boldsymbol{x})) &\leq L(f^{\star}(\cdot;\epsilon'(\boldsymbol{x})), \lambda^{\star}(\boldsymbol{x};\epsilon(\boldsymbol{x}));\epsilon(\boldsymbol{x})) \\
&=\mathbb{E}_{(\boldsymbol{x}, y) \sim \mathfrak{D}} \: \Big[ \ell (f^{\star}(\boldsymbol{x};\epsilon'(\boldsymbol{x})), y)  + \lambda^{\star}(\boldsymbol{x};\epsilon(\boldsymbol{x}))( \ell'(f^{\star}(\boldsymbol{x};\epsilon'(\boldsymbol{x})), y)  - \epsilon(\boldsymbol{x}) )\Big]
\end{align*} Now, since $f^{\star}(\cdot;\epsilon'(\boldsymbol{x}))$ is \emph{optimal} for constraint bounds given by $\epsilon'(\boldsymbol{x})$ and complementary slackness holds, we have $\mathbb{E}_{(\boldsymbol{x}, y) \sim \mathfrak{D}} \: \Big[ \ell (f^{\star}(\boldsymbol{x};\epsilon'(\boldsymbol{x})), y)  \Big] = P^\star(\epsilon'(\boldsymbol{x}))$. Moreover, since $f^{\star}(\cdot;\epsilon'(\boldsymbol{x}))$ is feasible for constraint bounds given by $\epsilon'(\boldsymbol{x})$, we have $\ell'(f^{\star}(\boldsymbol{x};\epsilon'(\boldsymbol{x})), y)  \leq \epsilon'(\boldsymbol{x}), ~ \mathfrak{D}_{\mathbf{x}} \text {-a.e.}$ Combining the above, we get
\begin{align*}
P^\star(\epsilon(\boldsymbol{x})) &\leq P^\star(\epsilon'(\boldsymbol{x})) + \mathbb{E}_{(\boldsymbol{x}, y) \sim \mathfrak{D}} \: \Big[ \lambda^{\star}(\boldsymbol{x};\epsilon(\boldsymbol{x}))( \epsilon'(\boldsymbol{x})  - \epsilon(\boldsymbol{x}) )\Big] \\
& = P^\star(\epsilon'(\boldsymbol{x})) + \langle \lambda^{\star}(\boldsymbol{x};\epsilon(\boldsymbol{x}))( \epsilon'(\boldsymbol{x})  - \epsilon(\boldsymbol{x}) )\rangle,
\end{align*} or eqivalently,
\begin{align*}
 P^\star(\epsilon'(\boldsymbol{x})) - P^\star(\epsilon(\boldsymbol{x})) \geq \langle -\lambda^{\star}(\boldsymbol{x};\epsilon(\boldsymbol{x}))( \epsilon'(\boldsymbol{x})  - \epsilon(\boldsymbol{x}) )\rangle,
\end{align*} which exactly matches the definition of the Fréchet subdifferential in Definition 3.1, hence completing the proof.

This shows that $ -\lambda^{\star}(\boldsymbol{x}) $ is a sub-gradient of $P^*(\epsilon(\boldsymbol{x}))$. Let $\partial P^*$ be the sub-differential of $P^*$ (i.e., the set of all sub-gradients). By definition, $P^*$ is differentiable if its sub-differential is a singleton: $\partial P^* =  \{ -\lambda^{\star}(\boldsymbol{x}) \}$. This holds in our problem if the Lagrangian minimizers are unique, which is the case for a strongly convex objective. However, it does not hold in the general case. In \cite{ShapiroDifferentiability}, right-side differentiability is shown by further assuming that $\exists \alpha > P^*$ and a compact set $S$ such that set of feasible points where the objective does not exceed $\alpha$ is contained in S.

\iffalse %moved to main body
\section{Additional Experiments}
\label{add_exps}

\begin{figure}[H]
     \centering
     \begin{subfigure}[b]{0.8\textwidth}
         \centering
         \includegraphics[width=.8\textwidth]{images/legend.pdf}
     \end{subfigure}
     \begin{subfigure}[b]{0.45\textwidth}
         \centering
         \includegraphics[width=.99\textwidth]{images/classification/STL101000.pdf}
     \end{subfigure}
     \begin{subfigure}[b]{0.45\textwidth}
         \centering
         \includegraphics[width=.99\textwidth]{images/classification/SVHN200.pdf}
     \end{subfigure}
     \begin{subfigure}[b]{0.45\textwidth}
         \centering
         \includegraphics[width=.99\textwidth]{images/classification/CIFAR10_1000.pdf}
     \end{subfigure}
         \begin{subfigure}[b]{0.45\textwidth}
         \centering
         \includegraphics[width=.99\textwidth]{images/classification/MNIST_200.pdf}y
     \end{subfigure}
\end{figure}
\fi

\section{The Statistical Bias Induced by Active Sampling} 
\label{app:biasAL}

As mentioned in Section~\ref{sec:biasAL}, when performing several active learning iterations, we undertake a biased version of \eqref{eq:batch_activelearning}:
\begin{equation*}
\tag{b-BAL}
\mathcal{B}^* = \argmin _{\mathcal{B} \subseteq \mathcal{U}_t \: : \: | \mathcal{B} | \leq b} \; \min _{f \in \mathcal{F}} \mathbb{E}_{(\mathbf{x}, y) \sim \mathfrak{A}^{(t)}}\left[\ell\left(f(\mathbf{x}; \mathcal{L}_{t} \cup \mathcal{B}), y
\right)\right],
\end{equation*} where $\mathfrak{A}^{(t)}$ represents the biased/shifted distribution underlying the actively sampled set $\mathcal{L}^{(t)}$.. This makes the learnt predictor $f(\mathbf{x}; \mathcal{L}^{(t)} \cup \mathcal{B})$ sub-optimal for the natural data distribution $\mathfrak{D}$.

We focus our attention on the setting with $\ell = \ell'$. Observe that if $\mathfrak{D}_{\mathbf{x}}$ and $\mathfrak{A}_{\mathbf{x}}$ have the same support, then the feasibility formulation of the inner minimization is the same whether we consider the natural or biased data distribution. That is, 

\begin{align*}
 & P^{\star} = \min_{f \in \mathcal{F}}  \quad 0 & \iff \quad &   P^{\star} = \min_{f \in \mathcal{F}}  \quad 0  \\
& \text{s.t.} \quad \ell \left(f(\boldsymbol{x}), y\right) \leq \epsilon(\boldsymbol{x}), \quad \mathfrak{D}_{\boldsymbol{x}} \text {--a.e. }   & & \text{s.t.}  \quad \ell \left(f(\boldsymbol{x}), y\right) \leq \epsilon(\boldsymbol{x}), \quad \mathfrak{A}_{\boldsymbol{x}} \text {--a.e. }
\end{align*}

This is because, \begin{align*}
    & \ell( f(\mathbf{x}), y ) \leq \epsilon(\mathbf{x}) \quad \mathfrak{D}_{\boldsymbol{x}} \text {--a.e. } &  \iff  & \quad  \ell( f(\mathbf{x}), y ) p_{\mathfrak{D}_{\boldsymbol{x}}}(\mathbf{x}) \leq \epsilon(\mathbf{x}) p_{\mathfrak{D}_{\boldsymbol{x}}}(\mathbf{x}) \quad \forall \mathbf{x} \\
 & \ell( f(\mathbf{x}), y ) \leq \epsilon(\mathbf{x}) \quad \mathfrak{A}_{\boldsymbol{x}} \text {--a.e. }  & \iff & \quad \ell( f(\mathbf{x}), y ) p_{\mathfrak{D'}_{\boldsymbol{x}}}(\mathbf{x}) \leq \epsilon(\mathbf{x}) p_{\mathfrak{A}_{\boldsymbol{x}}}(\mathbf{x}) \quad \forall \mathbf{x}
\end{align*}  

Notice that for both distributions, the constraints are \emph{only} enforced when the respective density is non-zero. When the density is exactly zero, the constraint is satisfied trivially. Therefore, if $$\{ x \: : \: p_{\mathfrak{D}_{\boldsymbol{x}}}(\mathbf{x}) > 0 \} =  \{ x \: : \: p_{\mathfrak{A}_{\boldsymbol{x}}}(\mathbf{x}) > 0 \}$$ then the feasible set is the same for both problems. Since the objectives also coincide, the problems are identical. \\
In fact, we can replace the objective in the above problems by any functional $R(f)$ that does not depend on the data distribution (e.g: regularizers or complexity measures such as $ \| f \| $) and the equivalence still holds. When setting a regularizing functional as the objective, the resulting dual problem is:

\begin{equation*}
\label{eq:dfsl}
\tag{D-FSL}
D^{\star} =\max _{\lambda\in\Lambda} \; \min _{f \in \mathcal{F}} R(f) +  \lambda(\boldsymbol{x})( \ell(f(\boldsymbol{x}), y)  - \epsilon(\boldsymbol{x}) )
\end{equation*}

We know that the primal problem of \eqref{eq:dfsl} is agnostic to the underlying data distribution. If we further assume the existence of a striclty feasible solution, strong duality holds (see Appendix \ref{appx:duality}). Thus, the solution of \eqref{eq:dfsl} is also unimpacted by the distribution shift induced by active sampling.

Observe that the dual problem in the original ALLY formulation, with $\ell = \ell'$, is:
\begin{equation*}
\tag{D-CSL}
D^{\star}=\max _{\lambda\in\Lambda} \; \min _{f \in \mathcal{F}} \: (\lambda(\boldsymbol{x}) + \mathbf{1} ) \ell(f(\boldsymbol{x}), y)   - \lambda(\boldsymbol{x}) \epsilon(\boldsymbol{x}) 
\end{equation*}

Note that the dual problems \eqref{eq:D-CSL} and \eqref{eq:dfsl} differ in two aspects: the presence of a regularizer $R(f)$ and a unit shift in the dual variable function $\lambda(\mathbf{x})$. However, a unit shift in the dual variable function does not affect the relative impact of each sample on the expected loss. That is, the ranking of informativeness scores remains unchanged. Thus, assuming the supports of  $\mathfrak{D}_{\mathbf{x}}$ and $\mathfrak{A}_{\mathbf{x}}$ coincide, ALLY is equivalent to a statistically consistent active learning method modulus the use of a dual function regularizer. 

\section{Empirical Primal-Dual Learning Procedure in Algorithm~\ref{alg:pd}}
\label{app:convergence}

Even in the case were \eqref{eq:cl_formulation} and \eqref{eq:D-CSL} are strongly dual, in practice we undertake the empirical version of \eqref{eq:D-CSL} due to the challenges mentioned in Section~\ref{empirical_csl} (i.e: $\mathfrak{D}$ is unknown and $\mathcal{F}$ is infinite dimensional). This requires introducing a \emph{parameterization} of the hypothesis class $\mathcal{F}$ as $\mathcal{P} = \{f_{\boldsymbol{\theta}} \, | \, \boldsymbol{\theta} \in \Theta \}$ and replacing expectations by sample means. In the following section, we present a high-level overview of some results from \cite{LuizPAC} on the implications of these two changes and the solution yielded by Algorithm PDCL.

\subsection{Approximation Error}
On one hand, approximating the function class $\mathcal{F}$ with a finite dimensional parametrization 
$\mathcal{P}$ transforms the dual problem \eqref{eq:D-CSL} into:

\begin{equation*}
\label{eq:D-CSL-tilde}
\tag{$\Tilde{D}$-CSL}
\Tilde{D^{\star}}=\max _{\lambda\in\Lambda} \; \min _{\theta \in \Theta} \Tilde{L}\left(f_{\theta}, \lambda(\boldsymbol{x}) \right),
\end{equation*}
where $$ \Tilde{L}\left(f_{\theta}, \lambda(\boldsymbol{x}) \right) =  \mathbb{E}_{(\boldsymbol{x}, y) \sim \mathfrak{D}} \: \Big[ \ell (f_{\theta}(\boldsymbol{x}), y)  + \lambda(\boldsymbol{x})( \ell'(f_{\theta}(\boldsymbol{x}), y)  - \epsilon(\boldsymbol{x}) )\Big]. $$

Assuming that $\mathcal{P}$ is PAC learnable and that there is $\nu > 0$ such that for each $f \in \mathcal{F}$ there exists $f_{\theta} \in P$ that satisifes $\supremum_{\boldsymbol{x} \in \mathcal{X}} | f_{\theta}(\boldsymbol{x}) - f(\boldsymbol{x}) | \leq \nu$,  $f_{\theta^{\star}}$ is a \emph{near optimal solution} of \eqref{eq:cl_formulation} \cite{LuizPAC} (Proposition 2).

Note that the assumption mentioned above is connected to the richness of the parameterization $\mathcal{P}$. For instance, $\mathcal{F}$ can denote $C^0$ (i.e: space of continuous functions) and $\mathcal{P}$ be a neural network, which meets the universal approximation assumption \cite{univapprox}.

\subsection{Estimation Error}

On the other hand, approximating expectations by their sample means transforms the statistical dual problem \eqref{eq:D-CSL-tilde} into its empirical counterpart \eqref{eq:D-CERM}:
\begin{align}
\hat{D^{\star}}=\max_{\boldsymbol{\lambda} \geq \mathbf{0}} \; \min _{\boldsymbol{\theta} \in \Theta} \hat{L}(\boldsymbol{\theta}, \boldsymbol{\lambda}), \tag{D-CERM}
\end{align}

This modification creates a difference between the optimal value of the parametrized dual problem and the optimal value of the empirical dual problem. However, assuming that the losses in question are $[0, B]$-valued and M-Lipschitz continuous -conditions which can be satisfied in the case of Cross-Entropy and Mean-Squared Error by setting bounds- this difference is bounded by a constant. In addition, this constant depends on the number of samples, the VC dimension of the parametrization  $\mathcal{P}$, $B$ and $M$. See \cite{LuizPAC} (Proposition 3) for a more detailed analysis.

In \cite{LuizPAC} (Theorem 3), the approximation and estimation errors are combined, and the sub-optimality of Algorithm PDCL with respect to \eqref{eq:cl_formulation} is bounded. 

\section{Additional Experimental Settings}
\label{app:exp_setts}

We do not use warm-starting or data augmentation. We employ early stopping on a validation set from the unlabelled pool and run each experiment five times with different random seeds. The patience in the early stopping callback was set to 2 epochs. We report the mean and standard deviations across the different seeds.  We smooth the curves using exponential moving average for clarity and focus on discriminative regions of the learning curves. The dropout parameters for the dual prediction head were determined by cross-validation and are set to $0.3$ and $0.25$. The parameters of the neural networks are updated once at each iteration ($T_p=1$) using the ADAM optimizer \cite{adam}, and a learning rate of $\eta_p = 0.001$ for SVHN, STL and CIFAR and $\eta_p = 0.005$ on MNIST. We adapt the architecture for STL-10 images ($96\times96\times3$) and split the dataset so as to have 11000 train samples and 2000 test samples (the original version has 5000 training samples and 8000 testing samples). This architecture has the same embedding size as ResNet-18 (i.e: 512).  The dual variables are updated with stochastic gradient ascent and a learning rate of $\eta_d = 0.05$. During  primal and dual steps we use a diminishing update rule for the step size. The number of layers of the dual variable prediction head was set by cross validation. All experiments were carried out on a system with the following specifications: Ubuntu 20.04, AMD Threadripper 3960X CPU and RTX 3090 GPU. 

\section{Ablation on the Number of Clusters}
\label{app:ablation_clust}

We perform an ablation study on the number of clusters used by the $k$-MEANS algorithm. We use the MNIST dataset and the MLP architecture described in section \ref{sec:experiments}. As shown in Figure \ref{fig:ablclust}, the performance of ALLY improves as the number of clusters grows. This suggests that, in this setting, prioritizing diversity over individual sample informativeness is beneficial in terms of average test accuracy.

\begin{figure}[H]
    \centering
    \includegraphics[width=.4\columnwidth]{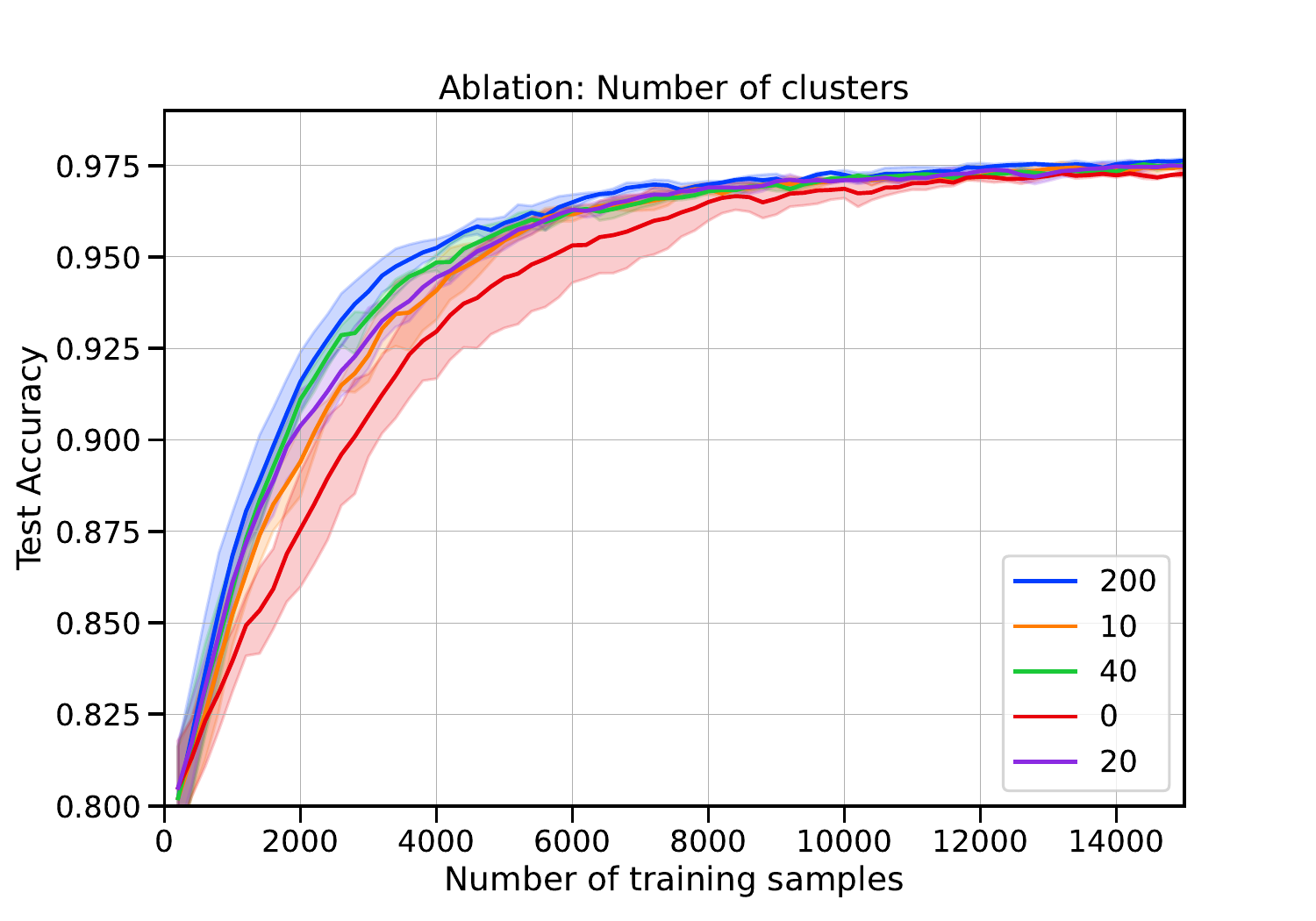}
    \caption{Ablation study on the number of clusters $k$ used in ALLY.}  
    \label{fig:ablclust}
\end{figure}

\section{Experiment on Tiny ImageNet}

Figure~\ref{fig:tiny_imagenet} compares the performance of ALLY with baselines methods on the Tiny ImageNet dataset \cite{tiny_im}, which is a subset of the ImageNet dataset (ILSVRC2012) consisting of $64 \times 64 \times 3$ images categorized in 200 classes.

\begin{figure}[H]
    \centering
    \includegraphics[width=.49\columnwidth]{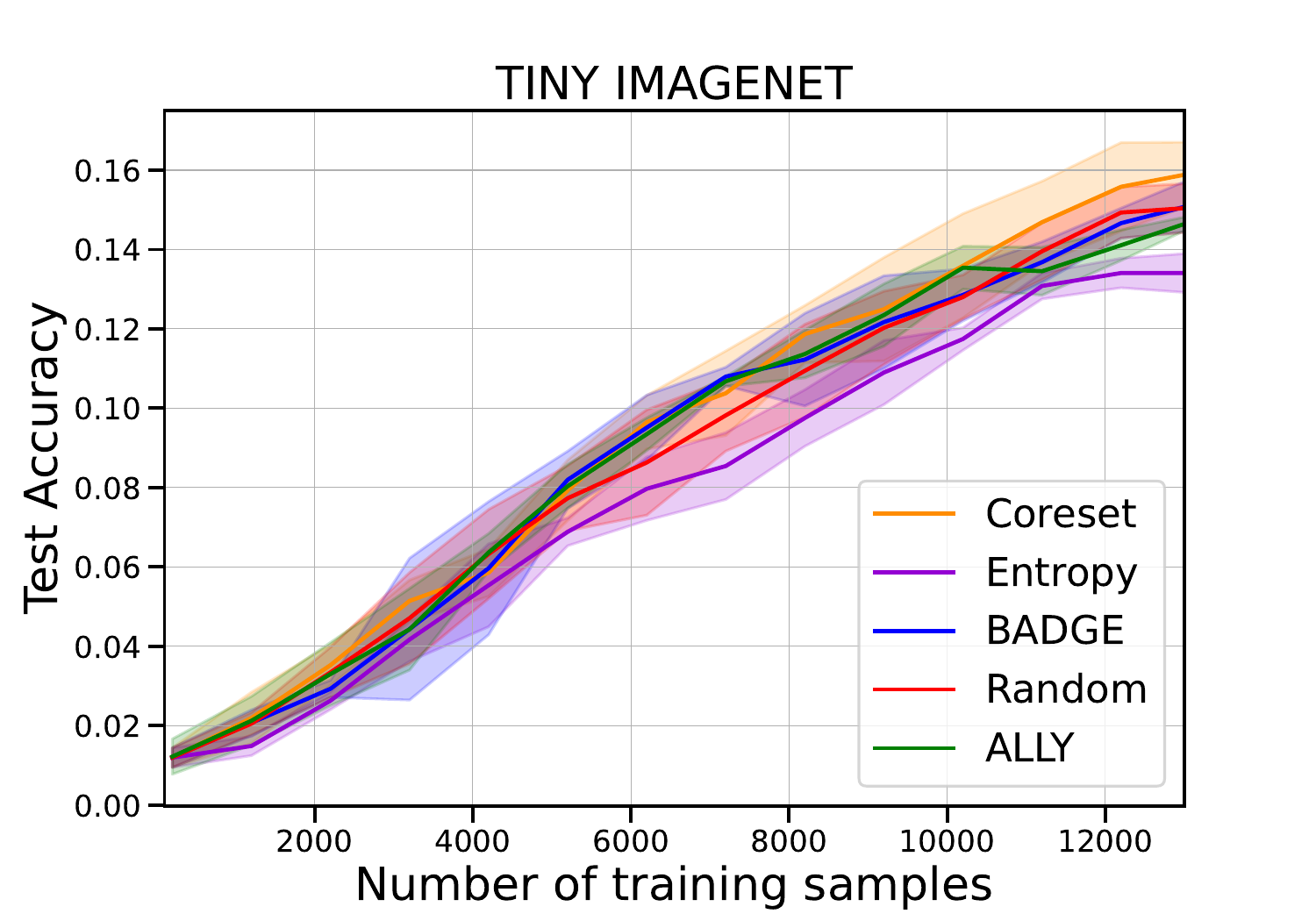}
    \caption{Empirical evaluation on the Tiny ImageNet dataset with $b=1000$.}
    \label{fig:tiny_imagenet}
\end{figure}